\documentclass[letterpaper]{article} %
\usepackage[preprint]{aaai2027}  %
\usepackage[hyphens]{url}  %
\usepackage{graphicx} %
\usepackage{natbib}  %
\usepackage{caption} %
\usepackage{algorithm}
\usepackage{algorithmic}

\usepackage{newfloat}
\usepackage{listings}
\DeclareCaptionStyle{ruled}{labelfont=normalfont,labelsep=colon,strut=off} %
\floatstyle{ruled}
\newfloat{listing}{tb}{lst}{}
\floatname{listing}{Listing}

\usepackage{booktabs}
\usepackage{amsmath}
\usepackage{array}
\usepackage{xspace}
\usepackage{amssymb}
\usepackage{amsthm}
\usepackage{mathtools}

\usepackage{thmtools}
\newif\iflongversion
\longversionfalse

\newtheorem{definition}{Definition}
\newtheorem{theorem}{Theorem}
\newtheorem{lemma}{Lemma}
\newtheorem{corollary}{Corollary}

\newtheorem{proposition}{Proposition}

\newcommand{\ltl}{\text{LTL}\xspace}
\newcommand{\ltlf}{\text{LTLf}\xspace}
\newcommand{\ltlfplus}{\text{LTLf+}\xspace}

\newcommand{\ltlfU}{\textsf{U}}
\newcommand{\ltlfM}{\textsf{M}}
\newcommand{\ltlfW}{\textsf{W}}
\newcommand{\ltlfX}{\textsf{X}}
\newcommand{\ltlfN}{\textsf{N}} %
\newcommand{\ltlfG}{\textsf{G}}
\newcommand{\ltlfR}{\textsf{R}}
\newcommand{\ltlfF}{\textsf{F}}
\newcommand{\ltlftrue}{\textsf{tt}}
\newcommand{\ltlffalse}{\textsf{ff}}

\newcommand{\ap}{\textsf{AP}}
\newcommand{\last}{\mathit{last}\xspace}
\newcommand{\buchi}{B\"uchi\xspace}

\newcommand{\mup}{\mu(\varphi)}   %
\newcommand{\nup}{\nu(\varphi)}   %
\newcommand{\BGF}{\mup}
\newcommand{\BFG}{\nup}

\newcommand{\rewM}[1]{[#1]}                     %
\newcommand{\fgrw}[1]{[#1]_{\nu}}                        %
\newcommand{\gfrw}[1]{[#1]_{\mu}}                     %

\newcommand{\rewMN}[1]{\langle#1\rangle}        
\usepackage{todonotes}

\newcommand{\spot}{{$\mathsf{SPOT}$}\xspace}

\newcommand{\strix}{{$\mathsf{Strix}$}\xspace}
\newcommand{\semml}{{$\mathsf{SemML}$}\xspace}

\newcommand{\normalisation}{$\Delta_2$-normalization\xspace}

\newcommand{\modelsinf}{\models_{\mathsf{inf}}}

\newcommand{\modelsfin}{\models_{\mathsf{fin}}}

\newcommand{\fin}{\mathsf{fin}}
\newcommand{\finX}{\mathsf{fin}_{\ltlfX}}
\newcommand{\finN}{\mathsf{fin}_{\ltlfN}}

\newcommand{\chop}{\mathsf{chop}}

\newcommand{\tight}{\mathsf{tight}}

\newcommand{\proj}{\mathsf{proj}}

\newcommand{\safetyleaf}[1]{\underbrace{#1}_{\text{safety leaf}}}

\title{Infinite Trace Objectives with Finite Trace Techniques: Translating LTL to LTLf+}
\author {
    Christoph Weinhuber\textsuperscript{\rm 1}\corresponding,
    Maximilian Prokop\textsuperscript{\rm 2,3},
    Giuseppe De Giacomo\textsuperscript{\rm 1}, 
    Moshe Y. Vardi\textsuperscript{\rm 4}
}
\affiliations {
    \textsuperscript{\rm 1}University of Oxford\\
    \textsuperscript{\rm 2}Technical University of Munich\\
    \textsuperscript{\rm 3}Masaryk University Brno\\
    \textsuperscript{\rm 4}Rice University\\
}

\begin{document}

\maketitle

\begin{abstract}
Linear Temporal Logic (LTL) is one of the most widely adopted languages for specifying temporal extended objectives in AI, with applications ranging from reactive synthesis to stochastic planning in Markov decision processes and reinforcement learning. Traditionally, solving any of these problems requires translating the LTL specification to a nondeterministic automata on infinite words and then determinizing it, a step that is notoriously difficult in theory and in practice. Recent work has introduced LTLf+, which lifts the finite-trace logic LTLf to infinite traces. LTLf+ has the same expressive power as LTL, yet it retains most of the crucial advantages of its base logic LTLf. Most reasoning in LTLf+ rests on finite automata on finite words, for which we have not only a canonical minimal representation but also an efficient determinization procedure. In this work we present the first translation from LTL to LTLf+. We first normalize an LTL formula into the syntactic reactivity fragment of the Manna-Pnueli hierarchy, to create the general fragment-based shape of LTLf+. We then present linear translations for each individual component of that fragment. As a consequence of this translation, the expanding body of techniques developed for LTLf+ now becomes available to many AI problems currently formulated in LTL.  We further show that this comes at no asymptotic cost, as the pipeline from LTL to automaton via LTLf+ remains doubly exponential.
\end{abstract}

\section{Introduction}\label{sec:intro}

Linear Temporal Logic (\ltl) \citep{DBLP:conf/focs/Pnueli77} is one of the most widely adopted specification languages for reasoning about the temporal behavior of dynamic systems.
It plays a central role in model checking \citep{DBLP:books/daglib/0020348}, reactive synthesis \citep{DBLP:conf/popl/PnueliR89}, supervisory control \citep{DBLP:journals/deds/EhlersLTV17}, and AI planning for temporally extended goals \citep{DBLP:journals/amai/BacchusK98,DBLP:conf/ijcai/GiacomoR18,DBLP:conf/aips/CamachoBM19}.
In robotics, \ltl specifications express complex mission objectives such as persistent surveillance and conditional responses to environmental events \citep{DBLP:journals/automatica/FainekosGKP09}.
In business process management, \ltl captures compliance constraints and procedural requirements \citep{DBLP:conf/bpm/MaggiMWA11}.
Across all these domains, \ltl provides a natural and declarative way to specify \emph{what} a system should do, without prescribing \emph{how}.
Despite its wide applicability, the computational cost of strategic reasoning tasks, such as \ltl synthesis, remains a barrier.
The traditional pipeline requires translating an \ltl formula into a nondeterministic \buchi automaton on infinite words and determinizing it.
But the determinization step is notoriously difficult both in theory and in practice \citep{DBLP:conf/stacs/Vardi07}, and has long limited the scalability of \ltl-based synthesis methods.

To sidestep determinization difficulties, state-of-the-art \ltl synthesis tools exploit a normal form that decomposes the formula into a Boolean combination of simpler formulas \citep{DBLP:journals/jacm/EsparzaRS24}. 
Each component can be translated into a deterministic \buchi or co-\buchi automaton, and their Boolean combination yields directly a deterministic Emerson-Lei automaton \citep{DBLP:journals/scp/EmersonL87}, on which synthesis is then solved.
This normal form is based on Manna and Pnueli's safety-progress hierarchy \citep{DBLP:conf/podc/MannaP89}, which distinguishes six basic classes of temporal properties: safety (nothing bad ever happens), guarantee (something good eventually happens), their boolean combinations, called obligations, recurrence (something good happens infinitely often), persistence (something bad happens only finitely often), and their boolean combinations, called reactivity. 
While the normalization is worst-case exponential, it is now used by the best-performing \ltl synthesis tools, from \strix \citep{DBLP:conf/cav/MeyerSL18} to \semml \citep{DBLP:conf/tacas/KretinskyMPZ25,DBLP:journals/corr/abs-2604-24102}, which have dominated the \ltl synthesis tracks of SYNTCOMP in the recent years.

Recently a different route has emerged in AI, around the logic \ltlfplus \citep{DBLP:conf/ijcai/AminofGRV25}.  \ltlfplus  builds on the finite-trace logic \ltlf \citep{DBLP:conf/ijcai/GiacomoV13,DBLP:conf/ijcai/GiacomoV15}, which interprets the syntax of \ltl over finite traces. Notably 
\ltlf formulas can be effectively translated into deterministic finite automata (DFAs).
Furthermore, in contrast to their infinite-word counterparts, DFAs can be efficiently minimized.
As a result \ltlf synthesis is reaching the performance and scalability of model checking \cite{DBLP:conf/wia/DuretLutzZPGV25}, which in turn is an industrial reality \cite{PSL,SVA}. 

\ltlfplus lifts \ltlf to infinite traces, again through the Manna-Pnueli hierarchy.
It places one of four prefix quantifiers over an \ltlf formula, requiring it to hold on some prefix ($\exists$, guarantee), on every prefix ($\forall$, safety), on infinitely many prefixes ($\forall\exists$, recurrence), or on all but finitely many prefixes ($\exists\forall$, persistence).
These quantifiers match exactly the corresponding classes of the Manna-Pnueli hierarchy, with obligations and reactivity captured by their Boolean combinations.
\ltlfplus has the full expressive power of \ltl, yet its reasoning problems are mostly solved on DFAs built from the finite-trace \ltlf subformulas.

\paragraph{Contribution.}
In this paper, we show that \ltl, once put in normal form ~\citep{DBLP:journals/jacm/EsparzaRS24}, can be translated into \ltlfplus in \emph{linear time}.

We thereby close a gap in \citep{DBLP:conf/ijcai/AminofGRV25} by lifting the existential result to a constructive one.
Further, we pave the way for a new synthesis pipeline, that first converts \ltl to \ltlfplus and then continues through the \ltlfplus synthesis pipeline.
In particular, this allows for DFA-based techniques (e.g.\ automaton minimization which previously was off limits for \ltl) to be applied to problems stated in \ltl.
More generally, we establish \ltlfplus as a practical intermediate language for arbitrary \ltl specifications, thereby enabling the growing ecosystem of DFA-based techniques developed for LTLf to be applied beyond finite-horizon specifications.
We thereby enable future algorithms developed for \ltlfplus to be applied to the full expressive power of \ltl without additional asymptotic cost.

\section{Preliminaries}\label{sec:prelim}

\paragraph{LTL.}
We use \ltl formulas \citep{DBLP:conf/focs/Pnueli77} in negation normal form (NNF) over a finite set $\ap$ of atomic propositions:
\begin{align*}
\varphi ::={} & \ltlftrue \mid \ltlffalse \mid p \mid \neg p
\mid \varphi \wedge \varphi \mid \varphi \vee \varphi
\mid \ltlfX\varphi 
\\
&{}\mid \varphi\, \ltlfU\, \varphi
\mid \varphi\, \ltlfW\, \varphi
\mid \varphi\, \ltlfM\, \varphi
\mid \varphi\, \ltlfR\, \varphi
\end{align*}
with $p \in \ap$, and the usual abbreviations 
$\ltlfF\varphi \equiv \ltlftrue\,\ltlfU\,\varphi$ and
$\ltlfG\varphi \equiv \varphi\,\ltlfW\,\ltlffalse$.
We write $\Sigma = 2^{\ap}$ for the alphabet.
On an infinite trace $w = w_0 w_1 \cdots \in \Sigma^\omega$ and a position $i \geq 0$, the satisfaction relation $w,i \modelsinf \varphi$ is defined inductively as follows:
\begin{align*}
w,i &\modelsinf p
&&\iff p \in w_i,
\\
w,i &\modelsinf \neg p
&&\iff p \notin w_i,
\\
w,i &\modelsinf \varphi_1 \land \varphi_2
&&\iff w,i \modelsinf \varphi_1
       \text{ and } w,i \modelsinf \varphi_2,
\\
w,i &\modelsinf \varphi_1 \lor \varphi_2
&&\iff w,i \modelsinf \varphi_1
       \text{ or } w,i \modelsinf \varphi_2,
\\
w,i &\modelsinf \ltlfX\varphi
&&\iff w,i{+}1 \modelsinf \varphi,
\\
w,i &\modelsinf \varphi_1 \ltlfU \varphi_2
&&\iff
\begin{multlined}[t][.53\columnwidth]
  \exists k \geq i.\, w,k \modelsinf \varphi_2 \ \text{and} \\
  \forall j \in [i,k).\; w,j \modelsinf \varphi_1,
\end{multlined}
\\
w,i &\modelsinf \varphi_1 \ltlfW \varphi_2
&&\iff
\begin{multlined}[t][.53\columnwidth]
  w,i \modelsinf \varphi_1 \ltlfU \varphi_2 \ \text{or} \\
  \forall j \geq i.\; w,j \modelsinf \varphi_1,
\end{multlined}
\\
w,i &\modelsinf \varphi_1 \ltlfM \varphi_2
&&\iff
\begin{multlined}[t][.53\columnwidth]
  \exists k \geq i.\, w,k \modelsinf \varphi_1 \ \text{and} \\
  \forall j \in [i,k].\; w,j \modelsinf \varphi_2,
\end{multlined}
\\
w,i &\modelsinf \varphi_1 \ltlfR \varphi_2
&&\iff
\begin{multlined}[t][.53\columnwidth]
  w,i \modelsinf \varphi_1 \ltlfM \varphi_2 \ \text{or} \\
  \forall j \geq i.\; w,j \modelsinf \varphi_2.
\end{multlined}
\end{align*}
with $w, i \modelsinf \ltlftrue$ and $w, i \not \modelsinf \ltlffalse$.
We write $w \modelsinf \varphi$ for $w, 0 \modelsinf \varphi$, and $w[i..n]$ and $w[i...]$ for the infix $w_i \cdots w_n$ and the suffix $w_i w_{i+1} \cdots$.
We denote the set of all subformulas of $\varphi$
by $\mathit{sf}(\varphi)$. Although the grammar is in NNF, we write $\neg\varphi$ for the NNF dual of $\varphi$.

\paragraph{Fixpoint classification.}
Each temporal operator is a fixpoint. The operators $\ltlfF$, $\ltlfU$, and $\ltlfM$ are \emph{operators} with the \emph{least} fixpoint ($\mu$) that require a witness in a finite time. The operators $\ltlfG$, $\ltlfW$, and $\ltlfR$ are \emph{greatest} fixpoint ($\nu$) operators 
that hold unless violated within finite time.

\paragraph{Safety-progress hierarchy.}
The classes $\Sigma_i$ and $\Pi_i$ are defined inductively by which fixpoint types are allowed and at what alternation depth~\citep{DBLP:conf/podc/MannaP89,DBLP:conf/icalp/ChangMP92}.
Let $\Delta_0$ denote the Boolean combinations of $\ltlftrue$, $\ltlffalse$, atoms, negated atoms, allowing $\ltlfX$ as the only temporal operator.
$\Sigma_1$ (guarantee/co-safety) contains $\Delta_0$ and is closed under $\land$, $\lor$, $\ltlfX$, and the $\mu$-operators $\ltlfF$, $\ltlfU$, $\ltlfM$. 
$\Pi_1$ (safety) contains $\Delta_0$ and is closed under $\land$, $\lor$, $\ltlfX$, and the $\nu$-operators $\ltlfG$, $\ltlfW$, $\ltlfR$. 
$\Delta_1$ (obligation) is the Boolean closure of $\Sigma_1 \cup \Pi_1$.
$\Sigma_2$ (persistence) contains $\Pi_1$ and is closed under $\land$, $\lor$, $\ltlfX$, $\ltlfF$, $\ltlfU$, $\ltlfM$. This yields at most
one fixpoint alternation, with $\mu$-operators on the outside. 
$\Pi_2$ (recurrence) contains $\Sigma_1$ and is closed under $\land$, $\lor$, $\ltlfX$, $\ltlfG$, $\ltlfW$, $\ltlfR$. 
This yields at most one fixpoint alternation, with $\nu$-operators on the outside.
$\Delta_2$ (reactivity) is the Boolean closure of $\Sigma_2 \cup \Pi_2$. 
Moreover, every \ltl formula is equivalent to a formula in $\Delta_2$~\citep{DBLP:conf/icalp/ChangMP92}.

\paragraph{LTLf.} 
We follow \citet{DBLP:conf/ijcai/GiacomoV13} and define \ltl on finite traces (\ltlf) on a non-empty finite trace $u = u_0 \cdots u_n$ of length $n{+}1$.
We use $\modelsfin$ for the standard non-empty finite-trace satisfaction relation and write $\last = \ltlfN(\ltlffalse)$, which holds exactly at the last position of the trace.
\begin{align*}
u, k &\modelsfin p
&&\iff p \in u_k,
\\
u, k &\modelsfin \neg p
&&\iff p \notin u_k,
\\
u, k &\modelsfin \varphi_1 \land \varphi_2
&&\iff u, k \modelsfin \varphi_1
       \text{ and } u, k \modelsfin \varphi_2,
\\
u, k &\modelsfin \varphi_1 \lor \varphi_2
&&\iff u, k \modelsfin \varphi_1
       \text{ or } u, k \modelsfin \varphi_2,
\\
u, k &\modelsfin \ltlfX\varphi
&&\iff k < n \text{ and } u, k{+}1 \modelsfin \varphi,
\\
u, k &\modelsfin \ltlfN\varphi
&&\iff k = n \text{ or } u, k{+}1 \modelsfin \varphi,
\\
u, k &\modelsfin \varphi_1 \ltlfU \varphi_2
&&\iff
\begin{multlined}[t][.55\columnwidth]
  \exists m \in [k, n].\; u, m \modelsfin \varphi_2 \ \text{and} \\
  \forall j \in [k, m).\; u, j \modelsfin \varphi_1, \hspace{0.6em}
\end{multlined}
\\
u, k &\modelsfin \varphi_1 \ltlfW \varphi_2
&&\iff
\begin{multlined}[t][.55\columnwidth]
  u, k \modelsfin \varphi_1 \ltlfU \varphi_2 \ \text{or} \\
  \forall j \in [k, n].\; u, j \modelsfin \varphi_1, \hspace{0.5em}
\end{multlined}
\\
u, k &\modelsfin \varphi_1 \ltlfM \varphi_2
&&\iff
\begin{multlined}[t][.55\columnwidth]
  \exists m \in [k, n].\; u, m \modelsfin \varphi_1 \ \text{and} \\
  \forall j \in [k, m].\; u, j \modelsfin \varphi_2,
  \hspace{0.6em}
\end{multlined}
\\
u, k &\modelsfin \varphi_1 \ltlfR \varphi_2
&&\iff
\begin{multlined}[t][.55\columnwidth]
  u, k \modelsfin \varphi_1 \ltlfM \varphi_2 \ \text{or} \\
  \forall j \in [k, n].\; u, j \modelsfin \varphi_2.\hspace{0.4em}
\end{multlined}
\end{align*}
with $u, k \modelsfin \ltlftrue$ and $u, k \not \modelsfin \ltlffalse$.
We write $u \modelsfin \varphi$ for $u, 0 \modelsfin \varphi$ and write $\neg\varphi$ for the NNF dual of an \ltlf formula as well.

\paragraph{LTLf+.}
Following \citet{DBLP:conf/ijcai/AminofGRV25}, \ltlfplus extends \ltlf to infinite traces by quantifying over the non-empty prefixes $w[0..i] = w_0 \cdots w_i$ of an infinite trace $w$:
\begin{align*}
&w \modelsinf \forall\varphi &&\iff &&&\forall\, i \ge 0\colon w[0..i] \modelsfin \varphi, \\
&w \modelsinf \exists\varphi &&\iff &&&\exists\, i \ge 0\colon w[0..i] \modelsfin \varphi, \\
&w \modelsinf \forall\exists\varphi &&\iff &&&\forall\, i \ge 0,\; \exists\, j \ge i\colon w[0..j] \modelsfin \varphi, \\
&w \modelsinf \exists\forall\varphi &&\iff &&&\exists\, i \ge 0,\; \forall\, j \ge i\colon w[0..j] \modelsfin \varphi.
\end{align*}
We refer to any such formula as an \ltlfplus-\textit{leaf formula}.
An \ltlfplus formula is a Boolean combination of \ltlfplus leaf formulas
The four leaf types match the four basic classes of the hierarchy: $\forall\varphi$ is safety ($\Pi_1$), $\exists\varphi$ guarantee ($\Sigma_1$), $\forall\exists\varphi$ recurrence ($\Pi_2$), and $\exists\forall\varphi$ persistence ($\Sigma_2$). 
Full \ltlfplus is equivalent to reactivity ($\Delta_2$)~\citep{DBLP:conf/ijcai/AminofGRV25}.

\section{Normalization to $\Delta_2$ Normal Form}\label{sec:normalisation}
It noted above every \ltl formula can be translated to an equivalent formula that belongs to $\Delta_2$ \citep{DBLP:conf/lop/LichtensteinPZ85,DBLP:conf/icalp/ChangMP92}.
The original result, however, uses automata as intermediate representations and incurs a non-elementary blow-up. Recently, a pure formula rewriting system, which is only single exponential and much more effective, was proposed by \citet{DBLP:journals/jacm/EsparzaRS24}. As it plays a central role in our translation, we provide a brief account of this normalization procedure in this section.

\paragraph{Intuition.}
The normalization classifies the traces of the formula-to-be-normalized by the \emph{limit-behavior} they exhibit.
A limit-behavior is a pair $(M, N)$ of sets of limit formulas, where $M$ records which least fixpoint subformulas are satisfied infinitely often, and $N$ which greatest fixpoint subformulas will eventually be satisfied forever.
For every limit-behavior, we can rewrite the original formula assuming the current limit-behavior holds.
In particular, we can simplify arbitrary nesting of greatest and least fixpoint operators down to a single alternation (i.e.\ a $\Sigma_2$ or a $\Pi_2$ formula) when assuming a certain limit-behavior.
The normalized formula is then a disjunction over all limit-behaviors.
Each disjunct contains i) the simplified formula under the current limit-behavior and ii) a check that the current limit-behavior indeed holds.
Crucially, these checks fall into the classes $\Sigma_2$ or $\Pi_2$ which overall yields a Boolean combination of only $\Sigma_2$ or $\Pi_2$ formulas, or in other words, a syntactic $\Delta_2$ formula.
We restate the normalization below, culminating in Theorem~\ref{thm:normalisation}.

We begin by formally defining the set of limit formulas (the \textit{Basis}) $B$, from which $(M, N)$ is drawn. 
Each disjunct of the normalization
is then assembled from $(M, N)$ by three rewrites: $\varphi\fgrw{M}$ and $\varphi\gfrw{N}$ produce the persistence and recurrence checks of part ii) above, and $\varphi\rewM{M,N}$ the simplified formula of part i).

\begin{definition}[Basis {\citep{DBLP:journals/jacm/EsparzaRS24}}]\label{def:basis}
Let $\varphi$ be an \ltl formula and $\mathit{sf}(\varphi)$ its subformula set.
The \emph{basis} $B = \BGF \cup \BFG$ is defined as:
\begin{alignat*}{4}
\BGF
  &= \{\ltlfG\ltlfF\psi \mid \exists\, \varphi'.\; \varphi' \ltlfU \psi
  &{}\in{} \mathit{sf}(\varphi)
  &{}\lor{} \psi \ltlfM \varphi'
  &{}\in{} \mathit{sf}(\varphi)\},\\
\BFG
  &= \{\ltlfF\ltlfG\psi \mid \exists\, \varphi'.\; \psi \ltlfW \varphi'
  &{}\in{} \mathit{sf}(\varphi)
  &{}\lor{} \varphi' \ltlfR \psi
  &{}\in{} \mathit{sf}(\varphi)\}.
\end{alignat*}
\end{definition}

A tuple $(M, N)$ with $M \subseteq \BGF$ and $N \subseteq \BFG$ carries all the information needed to rewrite $\varphi$ under the assumption that the limit-behavior is correct. $M$ records the subformulas that hold infinitely often, while $N$ records those that eventually hold forever.
To check that $(M, N)$ actually holds, consider a subformula $\psi$ with $\ltlfF\ltlfG\psi \in N$. Its claim is that $\psi$ holds from some position onward. But $\ltlfF\ltlfG$ captures persistence only when its operand is a safety formula. On an operand that still contains least fixpoint operators it expresses something else, e.g. $\ltlfF\ltlfG(\ltlfF a)$ is just $\ltlfG\ltlfF a$, the recurrence of $a$. The following $\Pi_1$-rewrite resolves those operators in $\psi$ using $M$, leaving a safety formula $\psi\fgrw{M} \in \Pi_1$ on which $\ltlfF\ltlfG(\psi\fgrw{M})$ is a genuine persistence check.
Recurrence is dual, with the $\Sigma_1$-rewrite using $N$ to turn a subformula recorded by $M$ into the co-safety operand of $\ltlfG\ltlfF(\psi\gfrw{N})$.

\begin{definition}[$\Pi_1$-rewrite~\citep{DBLP:journals/jacm/EsparzaRS24}]\label{def:fg-rewrite}
For an operand $\psi$ with $\ltlfF\ltlfG\psi \in \BFG$ and a set $M \subseteq \BGF$, the formula rewrite $\psi\fgrw{M}$ leaves $\ltlftrue$, $\ltlffalse$ and literals unchanged, distributes over $\lor$, $\land$, $\ltlfX$, $\ltlfW$, $\ltlfR$, and resolves the $\mu$-operators $\ltlfF$, $\ltlfU$, $\ltlfM$:
\[
(\psi_1 \ltlfU \psi_2)\fgrw{M} =
\begin{cases}
\psi_1\fgrw{M}\, \ltlfW\, \psi_2\fgrw{M} & \text{if } \ltlfG\ltlfF\psi_2 \in M, \\
\ltlffalse & \text{otherwise.}
\end{cases}
\]
That is, if $M$ contains $\ltlfG\ltlfF\psi_2$, then $\ltlfU$ relaxes to~$\ltlfW$. The result of $\psi\fgrw{M}$ is a $\Pi_1$ formula with no least fixpoint operators. The $\ltlfM$ operator rewrite case is analogous, testing $\ltlfG\ltlfF\psi_1 \in M$, and relaxes to $\ltlfR$, and $\ltlfF$ is handled through $\ltlfF\psi \equiv \ltlftrue\,\ltlfU\,\psi$.
\end{definition}

\begin{definition}[$\Sigma_1$-rewrite~\citep{DBLP:journals/jacm/EsparzaRS24}]\label{def:gf-rewrite}
For an operand $\psi$ with $\ltlfG\ltlfF\psi \in \BGF$ and a set $N \subseteq \BFG$, the rewrite $\psi\gfrw{N}$ is dual to Definition~\ref{def:fg-rewrite}: it leaves $\ltlftrue$, $\ltlffalse$ and literals unchanged, distributes over $\lor$, $\land$, $\ltlfX$, $\ltlfF$, $\ltlfU$, $\ltlfM$, and resolves the $\nu$-operators $\ltlfG$, $\ltlfW$, $\ltlfR$:
\[
(\psi_1 \ltlfW \psi_2)\gfrw{N} =
\begin{cases}
\ltlftrue & \text{if } \ltlfF\ltlfG\psi_1 \in N, \\
\psi_1\gfrw{N}\, \ltlfU\, \psi_2\gfrw{N} & \text{otherwise.}
\end{cases}
\]
The $\ltlfR$ operator case is analogous, testing $\ltlfF\ltlfG\psi_2 \in N$, and strengthens to $\ltlfM$; the operator $\ltlfG$ is handled through $\ltlfG\psi \equiv \psi\,\ltlfW\,\ltlffalse$.
The result of $\psi\gfrw{N}$ is a $\Sigma_1$ formula with no greatest fixpoint operators.
\end{definition}

With the limit-behavior fixed, so that $(M, N)$ may be assumed to hold on the current trace, we can now apply the simplification to the formula $\varphi$ being normalized.

\begin{definition}[$\Sigma_2$-rewrite~\citep{DBLP:journals/jacm/EsparzaRS24}]\label{def:body-rewrite}
For $(M, N)$, the rewrite $\varphi\rewM{M,N}$ leaves $\ltlftrue$, $\ltlffalse$ and literals unchanged, distributes over $\lor$, $\land$, $\ltlfX$, $\ltlfF$, $\ltlfU$, $\ltlfM$, and resolves the $\nu$-operators $\ltlfG$, $\ltlfW$, $\ltlfR$, each keeping a single highlighted \emph{safety leaf}:
\begin{align*}
(\ltlfG\,\psi_1)\rewM{M,N} &= \psi_1\rewM{M,N}\, \ltlfU\, \safetyleaf{\ltlfG(\psi_1\fgrw{M})} \\
(\psi_1 \ltlfW \psi_2)\rewM{M,N} &= \psi_1\rewM{M,N}\, \ltlfU\, \bigl(\psi_2\rewM{M,N} \!\lor\! \safetyleaf{\ltlfG(\psi_1\fgrw{M})}\bigr) \\
(\psi_2 \ltlfR \psi_1)\rewM{M,N} &= \bigl(\psi_2\rewM{M,N}\!\lor\! \safetyleaf{\ltlfG(\psi_1\fgrw{M})}\bigr)
\ltlfM\, \psi_1\rewM{M,N}
\end{align*}
In each clause the safety leaf becomes $\ltlffalse$ unless $\ltlfF\ltlfG\psi_1 \in N$.
The result $\varphi\rewM{M,N}$ is always in~$\Sigma_2$.
\end{definition}

Intuitively, the $\Sigma_2$-rewrite flattens an arbitrarily nested \ltl formula, into a single fixpoint alternation.

\begin{restatable}[\normalisation~\protect\citep{DBLP:journals/jacm/EsparzaRS24}]{theorem}{normalisationthm}
\label{thm:normalisation}
Every \ltl formula $\varphi$ is equivalent to:
\begin{multline*}
\! \! \! \! \!
\bigvee_{\substack{M \subseteq \BGF \\ N \subseteq \BFG}} \! \! \!
\Big(\varphi\rewM{M,N} ~\wedge\!\!\! \bigwedge_{\ltlfF\ltlfG\psi \in N} \! \! \! \!\ltlfF\ltlfG(\psi\fgrw{M})
 ~\wedge \! \! \! \bigwedge_{\ltlfG\ltlfF\psi \in M} \! \! \!\!\ltlfG\ltlfF(\psi\gfrw{N})\Big).
\end{multline*}
The resulting normalized formula has length $2^{O(|\varphi|)}$.
\end{restatable}

\paragraph{Synthesis with Normal Form.}
The $\Delta_2$ normal form is the standard starting point for reactive synthesis from \ltl~\citep{DBLP:conf/cav/MeyerSL18,DBLP:conf/tacas/KretinskyMPZ25,DBLP:journals/corr/abs-2604-24102}.
To construct a deterministic automaton, we first translate the individual components (i.e.\ the $\Sigma_2$ bodies and the limit behavior formulae) for each disjunct of the normalization in Theorem~\ref{thm:normalisation} into deterministic Büchi and co-Büchi automata using the simple constructions from  \citet{DBLP:journals/jacm/EsparzaKS20}. This yields a boolean combination of deterministic Büchi and co-Büchi, or equivalently, a deterministic Emerson Lei automaton, where the state space is the synchronous product and the acceptance condition is the exact same boolean combination of the individual Büchi and co Büchi acceptances. As there are exponentially many deterministic Büchi and co-Büchi automata and each of them is doubly exponential in the original formulas size, the Emerson Lei automaton has $\big(2^{2^{O(|\varphi|)}}\big)^{2^{O(|\varphi|)}} = 2^{2^{O(|\varphi|)}}$ states. Intuitively, the exponentially many factors are accumulating inside the inner exponent rather than adding a further exponentiation. This means that, even though the normalization yields an exponentially large formula, converting that formula to a deterministic automaton does not become triply exponential. Instead, the normalization allows us to obtain a theoretically optimal and practically tractable LTL-to-deterministic-automaton translation that powers the state of the art synthesis tools Strix \citep{DBLP:conf/cav/MeyerSL18} and SemML \citep{DBLP:conf/tacas/KretinskyMPZ25,DBLP:journals/corr/abs-2604-24102}.

\section{Translating \ltl to \ltlfplus}\label{sec:route}
The four prefix quantifiers of \ltlfplus correspond exactly to the classes of the safety-progress hierarchy of \citet{DBLP:conf/podc/MannaP89}.
The existential quantifier $\exists$ expresses guarantee properties ($\Sigma_1$).
The universal quantifier $\forall$ expresses safety properties ($\Pi_1$).
Further, the alternating quantifiers $\forall\exists$ express recurrence properties ($\Pi_2$) and the symmetric combination $\exists\forall$ expresses persistence properties ($\Sigma_2$) \cite{DBLP:conf/ijcai/AminofGRV25}.
Thus, by design, every \ltlfplus{} formula is a boolean combination of formulae of these 4 fragments.

Consequently, when translating from arbitrary LTL formulas to LTLf+, we first use the normalization of Theorem \ref{thm:normalisation} to create the general shape of a boolean combination of the 4 fragments.
Recall that this yields a boolean combination of three types of formulas, namely:
\begin{enumerate}
    \item $\Pi_2$-formulas of the form $\ltlfG\ltlfF (\sigma_1)$, with $\sigma_1 \in \Sigma_1$
    \item $\Sigma_2$-formulas of the form $\ltlfF\ltlfG (\pi_1)$, with $\pi_1 \in \Pi_1$
    \item general $\Sigma_2$ formulas.
\end{enumerate}
where the former two are the checks of the guessed limit behaviors and the latter are the resulting $\Sigma_2$ bodies assuming the respective limit behavior holds.
To complete the translation, it suffices to translate each type of formula to its corresponding LTLf+ fragment.
In the remainder of this paper, we present exactly these translations.

\section{Translating Limit-Behavior Blocks}\label{sec:min-tight-completion}
Translating limit-behavior involves translating
$\Pi_2$-formulas of the form $\ltlfG\ltlfF (\sigma_1)$, with $\sigma_1 \in \Sigma_1$
and 
$\Sigma_2$-formulas of the form $\ltlfF\ltlfG (\pi_1)$, with $\pi_1 \in \Pi_1$.

We first convert the $\Pi_2$-formulas, which, following~\citet{DBLP:conf/ijcai/AminofGRV25}, receives the $\forall\exists$ quantifier, and construct a conversion $t \colon \Sigma_1 \to \ltlf$ such that for every infinite trace $w$,
$$
w \modelsinf \ltlfG \ltlfF (\sigma_1) \iff w \modelsinf \forall \exists\, t(\sigma_1).
$$

We start by observing that formulae of the $\Sigma_1$ fragment have something called a "good prefix".
In other words, all operators of the $\Sigma_1$ fragment want to be completed after a finite prefix, similar to their \ltlf versions.
This motivates to simply reinterpret the Sigma1 formula as an LTLf formula, which we formalize as follows:

\begin{definition}[Finite reinterpretation]
\label{def:fin}
For an \ltl formula $\sigma \in \Sigma_1$, the finite trace interpretation $\finX(\sigma)$ rewrites every \ltl next-operator to the $\ltlf$ strong-next-operator and leaves $\ltlftrue, \ltlffalse, \ltlfF, \ltlfU, \ltlfM, \land, \lor$ and literals unchanged.
\end{definition}

A similar notion was introduced by \cite{DBLP:journals/lmcs/CimattiGGMT23}, who already established a close link to the co-safety fragment $\Sigma_1$.
We restate and adapt a few helper lemmas and definitions from \cite{DBLP:journals/lmcs/CimattiGGMT23} and \cite{DBLP:journals/fmsd/KupfermanV01} that we require for our arguments.

\begin{definition}[Completion index]\label{def:completion-index}
Fix an infinite word $w \in \Sigma^\omega$ and a \ltl formula $\sigma \in \Sigma_1$. 
Then, let $c_\sigma(i)$ the completion index when starting from index $i$ with
$c_\sigma(i) = \min \{ n \geq i : w[i..n] \modelsfin \finX(\sigma) \}$ or $\infty \text{ if none exists}$.
We write $c_\sigma = c_\sigma(0)$ and get $c_\sigma(i) \geq i$.
\end{definition}

Intuitively, the completion index $c_\sigma(i)$ is the smallest end index $n$, such that the infix $w[i..n] \modelsfin \finX(\sigma)$.

\begin{lemma} \label{lem:monotonicity}
Let $\sigma \in \Sigma_1$. If $w[i..n] \modelsfin \finX(\sigma)$ then 
    $w[i..n+1] \modelsfin \finX(\sigma)$. 
    Hence $\{ n : w[i..n] \modelsfin \finX(\sigma) \} = [c_\sigma(i), \infty)$ or $\emptyset$, and 
    $w[i..n] \modelsfin \finX(\sigma) \iff n \geq c_\sigma(i)$.
\end{lemma}

Once a prefix satisfies $\finX(\sigma)$, so does every longer one. This follows by \citet[Lemma~3.4]{DBLP:journals/lmcs/CimattiGGMT23}, as $\finX(\sigma)$ lies in their co-safety fragment with strong-next only, rewriting $\sigma_1 \ltlfM \sigma_2$ as $\sigma_2 \ltlfU (\sigma_1 \land \sigma_2)$.

\begin{lemma}\label{lem:2-3}
    Let $\sigma \in \Sigma_1$. $w[i..n] \modelsfin \finX(\sigma) \implies w[i...] \modelsinf \sigma$.
\end{lemma}
Given an infinite trace, if $\finX(\sigma)$ already holds on one of its finite prefixes, then $\sigma$ holds on the infinite trace itself; this follows by the dual of \citet[Remark~5.3]{DBLP:journals/fmsd/KupfermanV01} or via 
\citet[Lemma~3.5]{DBLP:journals/lmcs/CimattiGGMT23}.

\begin{lemma}\label{lem:2-4}
    Let $\sigma \in \Sigma_1$. $w[i...] \modelsinf
    \sigma \implies c_\sigma(i) < \infty$.
\end{lemma}

This follows directly via the co-safety dual of \citet[Theorem~5.4]{DBLP:journals/fmsd/KupfermanV01} or via 
\citet[Lemma~3.5]{DBLP:journals/lmcs/CimattiGGMT23}.
Lemma $~\ref{lem:2-4}$ combined with Lemma $~\ref{lem:2-3}$, yields that an infinite trace satisfies $\sigma$ exactly when one of its finite prefixes satisfies $\finX(\sigma)$.

With these ingredients in place, we can start translating Boolean combinations of $\Sigma_1$ and $\Pi_1$, as we show below.

\begin{proposition}\label{prop:sigma1-translation}
     For every $\sigma \in \Sigma_1$ and $w \in \Sigma^\omega$,
     $w \modelsinf \sigma \iff w \modelsinf \exists(\finX(\sigma))$.
\end{proposition}
\begin{proof}
    If $w \modelsinf \sigma$, then $c_\sigma < \infty$ by Lemma~\ref{lem:2-4}, so $w[0..c_\sigma] \modelsfin \finX(\sigma)$ by Lemma~\ref{lem:monotonicity}, giving $w \modelsinf \exists(\finX(\sigma))$. 
    Conversely, if $w[0..n] \modelsfin \finX(\sigma)$ for some $n$, then $w \modelsinf \sigma$ by Lemma~\ref{lem:2-3}.
\end{proof}

\begin{proposition}\label{prop:pi1-translation}
     For every $\pi \in \Pi_1$ and $w \in \Sigma^\omega$, 
     $w \modelsinf \pi \iff w \modelsinf \forall(\neg \finX(\neg\pi))$.
\end{proposition}

\begin{proof}
    Note that $\pi \in \Pi_1$ implies $\neg \pi \in \Sigma_1$. 
        By the semantics of $\forall$, we get $w \modelsinf \forall(\neg\finX(\neg\pi))$ holds iff no prefix of $w$ satisfies $\finX(\neg\pi)$, which by Proposition~\ref{prop:sigma1-translation} applied to $\neg \pi$ holds iff $w \not\modelsinf \neg\pi$, that is, iff $w \modelsinf \pi$.
\end{proof}

Under the quantifiers $\exists$ and $\forall$, Propositions~\ref{prop:sigma1-translation} and ~\ref{prop:pi1-translation} suffice. 
However, to successfully put $\finX$ under the recurrence quantifier $\forall\exists$, we need a few extra ingredients.
To see why, consider $\sigma = a \wedge \ltlfX(b\,\ltlfU\,c)$ and $w = \{a\}\{b\}\{c\}\{\}^\omega$.
Here $\sigma$ holds from position $0$ and from no other, so $w \not\modelsinf \ltlfG\ltlfF(\sigma)$.
However, $w \modelsinf \forall\exists\finX(\sigma)$, as $w[0..i] \modelsfin \finX(\sigma)$ for every $i\geq2$.
The problem is that if $\sigma$ is satisfied once, this prefix is contained in every longer prefix.

We thus need to tie the satisfaction of $\sigma$ to the end of the prefix.
Reading $\sigma$ from the last position is of no use, though, as a formula whose satisfaction requires a later position, as $\sigma$ does through its next operator, is false there.
We instead leave the starting point free by writing $\ltlfF(\finX(\sigma))$ and ask that the satisfaction of $\sigma$ completes exactly at the end of the finite trace.
In other words, for a prefix $w[0..n]$ we want the $\ltlfF$-operator to choose an index $i$ s.t.\ the infix $w[i..n] \modelsfin \finX(\sigma)$.
We can achieve this by writing $\sigma' = \ltlfF(a \wedge \ltlfX(b\,\ltlfU\,(c \wedge \last)))$ where $\last = \ltlfN(\ltlffalse)$.
Indeed, this correctly rejects $w$, as only $w[0..2] \modelsfin \sigma'$ and no longer prefix does.
However, this is still too weak.
Consider $w' = \{a\}\{b,c\}^\omega$, where $\sigma$ again holds only from position $0$, while we again find infinitely many prefixes that satisfy $\sigma'$.
The problem here is that the infinitely many endpoints of the prefixes all correspond to the same starting point.
We need to ensure no such one-to-infinity correspondence between starting points (for the LTL side) and end points (for the LTLf+ side) can exist to ensure we have infinitely many of one if and only if we have infinitely many of the other.
To that end, we require that for any given starting point (chosen by the $\ltlfF$-operator) only the first possible endpoint at which the satisfaction of $\finX(\sigma)$ is completed, creates a valid prefix.
After that, the wrapping $\ltlfF$-operator needs to pick a new starting point.
We call this concept tight completion and formalize it as follows.

\begin{definition}[Chop]\label{def:chop}
For an \ltl formula $\sigma \in \Sigma_1$, we recursively define the resulting \ltlf formula $\chop(\sigma)$ as follows:

\begin{center}
\renewcommand{\arraystretch}{1.15}
\begin{tabular}{@{}l@{\;\;}c@{\;\;}l@{}}
$\chop(\varphi)$          & = & $\varphi$
    \quad for $\varphi \in \{\ltlftrue,\, \ltlffalse,\, p,\, \neg p\}$ \\
$\chop(\ltlfX(\sigma))$   & = & $\ltlfX(\neg \last \land \chop(\sigma))$ \\
$\chop(\ltlfF(\sigma))$   & = & $\ltlfF(\neg \last \land \chop(\sigma))$ \\
$\chop(\sigma_1 \ltlfU \sigma_2)$ & = &
    $\chop(\sigma_1) ~\ltlfU~ (\neg \last \land \chop(\sigma_2))$\\
$\chop(\sigma_1 \ltlfM \sigma_2)$ & = &
    $(\neg \last \land \chop(\sigma_1)) ~\ltlfM~ \chop(\sigma_2)$ \\
$\chop(\sigma_1 \lor \sigma_2)$ & = &
  $\chop(\sigma_1) \lor \chop(\sigma_2)$ \\
$\chop(\sigma_1 \land \sigma_2)$ & = &
  $\chop(\sigma_1) \land \chop(\sigma_2)$ \\

\end{tabular}
\end{center}
The size of $\chop(\sigma)$ is linear in the size of $\sigma$.
\end{definition}

\begin{restatable}[]{lemma}{choplemma}
\label{lem:chop}
Let $u$ be a finite trace with its last index $N \geq 1$, and let $p \in [0, N-1]$. For every $\sigma \in \Sigma_1$ we get 
$u[p..N] \modelsfin \chop(\sigma) \iff u[p..N-1] \modelsfin \sigma$.
\end{restatable}
Intuitively, $\chop(\sigma)$ holds on the full trace exactly when $\sigma$
holds on the trace with its last letter removed.
For example, let $\sigma = \ltlfF a$, so that $\chop(\sigma) = \ltlfF(\neg\last \land a)$.
Then for $u = \{\}\{a\}$ we have $u \modelsfin \finX(\sigma)$ and $u \not\modelsfin \chop(\sigma)$.
Further, for $v = \{\}\{a\}\{\}$ we have $v \modelsfin \finX(\sigma)$ and $v \modelsfin \chop(\sigma)$,
as satisfaction is decided before the last letter.

\begin{definition}[Tight completion]\label{def:tight}
Given an arbitrary \ltl formula $\sigma \in \Sigma_1$, we define
    $\tight(\sigma)$ = $\finX(\sigma) \land (\last \lor \neg \chop(\finX(\sigma)))$.
Furthermore, $\tight(\sigma)$ is linear in $\sigma$.
\end{definition}

Intuitively, $\tight(\sigma)$ captures the finite words that satisfy the formula $\finX(\sigma)$ \textit{tightly}, meaning that without the last letter they would not satisfy ($\neg\chop(\sigma)$).
It also captures satisfying words of length one (indicated by the formula $\last$ being true in the first step), as then the above reasoning is ill-defined.

\begin{restatable}[]{lemma}{tightcompllemma}\label{lemma:tight-compl}
For every $\sigma \in \Sigma_1$, $w \in \Sigma^\omega$, and all $0 \leq p \leq n$: $w[p..n] \modelsfin \tight(\sigma) \iff n = c_\sigma(p)$.
\end{restatable}

By Lemma~\ref{lem:monotonicity} the truth value of $\finX(\sigma)$ along the prefixes of $w$ changes at most once, from false to true, exactly at $c_\sigma$.

Tightening therefore reduces the satisfying prefixes from infinitely many to exactly one, namely $w[0..c_\sigma]$.
If no prefix satisfied $\finX(\sigma)$, none satisfies $\tight(\sigma)$ either.
Existence is therefore unaffected, and Propositions~\ref{prop:sigma1-translation} and~\ref{prop:pi1-translation} hold with $\tight$ in place of $\finX$ (Corollaries~\ref{corollary:E} and~\ref{corollary:A} in Appendix~\ref{appendix:min-tight-completion}).

Applied to the suffix $w[p...]$, Lemma~\ref{lemma:tight-compl} takes the positional form: $w[p..n] \modelsfin \tight(\sigma) \iff n = c_{\sigma}(p)$.

\begin{theorem}\label{theorem:AE}
For every \ltl formula $\sigma \in \Sigma_1$ and every $w \in \Sigma^\omega$,
$w \modelsinf \ltlfG \ltlfF (\sigma) \iff w \modelsinf \forall \exists (\ltlfF (\tight(\sigma)))$.
\end{theorem}

\begin{proof}
We first describe the prefixes that satisfy the \ltlf formula under the quantifier. By the semantics of $\ltlfF$, a prefix $w[0..n]$ satisfies $\ltlfF(\tight(\sigma))$ exactly when there is a position $p \leq n$ with $w[p..n] \modelsfin \tight(\sigma)$. By Lemma~\ref{lemma:tight-compl}, this holds exactly when $n = c_\sigma(p)$ for some $p \leq n$. The satisfying prefixes are therefore exactly those whose end index is the completion index of some position inside them.

($\Rightarrow$) Assume $w \modelsinf \ltlfG\ltlfF(\sigma)$, and let $m_1 < m_2 < \cdots$ be the infinitely many positions with $w[m_k...] \modelsinf \sigma$. By Lemma~\ref{lem:2-4}, every completion index $c_\sigma(m_k)$ is finite, and $c_\sigma(m_k) \geq m_k$ holds by definition. Set $n_k = c_\sigma(m_k)$. By the description above, every prefix $w[0..n_k]$ satisfies $\ltlfF(\tight(\sigma))$, with $p = m_k$ as the required position. Now let a deadline $i$ be given. Since the positions $m_k$ are unbounded, some $m_k$ satisfies $m_k \geq i$, and then $n_k \geq m_k \geq i$. Hence for every $i$ there is a satisfying prefix with end index at least $i$, which is $w \modelsinf \forall \exists (\ltlfF(\tight(\sigma)))$.

($\Leftarrow$) We prove the contrapositive. Assume $w \not\modelsinf \ltlfG\ltlfF(\sigma)$, and let $S$ be the set of positions $p$ with $w[p...] \modelsinf \sigma$, which is finite by assumption. Consider any $n$ such that $w[0..n]$ satisfies $\ltlfF(\tight(\sigma))$, and let $p \leq n$ be a position with $w[p..n] \modelsfin \tight(\sigma)$. Two facts hold for this position.
First, $p \in S$. The first conjunct of $\tight(\sigma)$ gives $w[p..n] \modelsfin \finX(\sigma)$, and Lemma~\ref{lem:2-3} gives $w[p...] \modelsinf \sigma$.
Second, $n = c_\sigma(p)$, by Lemma~\ref{lemma:tight-compl}. Each position of $S$ therefore determines exactly one end index.

Combining the two facts, the end index of every satisfying prefix lies in the set $\{\, c_\sigma(p) \mid p \in S \,\}$, which has at most $|S|$ elements. Only finitely many prefixes satisfy $\ltlfF(\tight(\sigma))$, so $w \not\modelsinf \forall \exists(\ltlfF(\tight(\sigma)))$.
\end{proof}

\begin{restatable}[]{theorem}{existsforalltheorem}\label{theorem:EA}
Given an \ltl formula $\pi \in \Pi_1$, we obtain:
$w \modelsinf \ltlfF \ltlfG (\pi) \iff w \modelsinf \exists\forall(\ltlfG(\neg \tight(\neg \pi)))$.    
\end{restatable}

\begin{proof}
    Since $\ltlfG(\neg \tight(\neg \pi)) \equiv \neg \ltlfF(\tight(\neg \pi))$ on finite traces, $w \modelsinf \exists \forall(\ltlfG(\neg \tight(\neg \pi))) \iff w \not \modelsinf \forall \exists (\ltlfF(\tight(\neg \pi)))$. By Theorem~\ref{theorem:AE} applied to $\neg \pi \in \Sigma_1$, the latter holds iff $w \not \modelsinf \ltlfG \ltlfF(\neg \pi)$, which is $w \modelsinf \ltlfF \ltlfG (\pi)$.
\end{proof}

\section{Translating $\Sigma_2$ Bodies}\label{sec:sigma2}
The remaining step to complete the translation is to convert the $\Sigma_2$ body.
First, recall that $\Sigma_2$ formulas are the closure of safety formulas under co-safety and Boolean operators.
This makes the persistence quantifier $\exists\forall$ the natural choice for the $\Sigma_2$ body, as already noted by \citet{DBLP:conf/ijcai/AminofGRV25}.
Further, recall that in the normalized formula each $\Sigma_2$ body is conjoined with a limit-behavior $(M, N)$, and every safety leaf $\ltlfG\pi$ of the body reappears in its disjunct as the conjunct $\ltlfF\ltlfG\pi$ (see Definition~\ref{def:body-rewrite} and Theorem~\ref{thm:normalisation}).
We use this context to show that the $\Sigma_2$ body is virtually equivalent to a $\Sigma_1$ formula.
Intuitively, for a trace that exhibits the limit-behavior, there is an index from which all safety leaves of the body hold, as they are guaranteed by the limit-behavior.
Satisfying a safety leaf then corresponds to reaching the index from which it is guaranteed to hold.
Informally speaking, this turns every safety leaf into a reachability objective, and the formula thus virtually behaves like a $\Sigma_1$ formula (note that this does not make a further rewrite possible, as the index differs from trace to trace and is in general unbounded).
Thus, Proposition~\ref{prop:sigma1-translation} (i.e.\ the translation through finite reinterpretation) becomes conceptually applicable.

For this, we need to lift the finite reinterpretation to $\Sigma_2$, which requires care with the next operators.
\ltl has a single next operator, whereas \ltlf has two, since on a finite trace a next letter need not exist while on an infinite trace it always does.
Each of the two \ltlf nexts has its natural fragment.
The strong next requires a next step to exist and satisfy its operand, which fits the co-safety fragment.
The weak next instead requires the operand to hold at the next step only when one exists, which fits the safety fragment.
Using the wrong one yields degenerate formulas, for instance $\ltlfF\ltlfN(\psi)$, a tautology on finite traces, and $\ltlfG\ltlfX(\psi)$, which is unsatisfiable.
The reinterpretation from \ltl to \ltlf therefore substitutes each next operator according to its context.
This is consistent with Proposition~\ref{prop:sigma1-translation}, where the formulas are $\Sigma_1$ and every next is therefore strong.

\begin{definition}[Extended finite reinterpretation]\label{def:fin-c}
For a context $c \in \{\ltlfX, \ltlfN\}$, the function $\fin_c$ traverses an NNF \ltl formula $\varphi$ 
and rewrites each next-operator, depending on the context, while leaving everything else in place.
In the $\ltlfX$-context a next becomes the strong next, in the $\ltlfN$-context the weak next.

Traversing an operand of $\ltlfF$, $\ltlfU$, $\ltlfM$ switches the context to $c=\ltlfX$. 
Traversing an operand of $\ltlfG$, $\ltlfW$, $\ltlfR$ switches it to $c=\ltlfN$. 
Literals, Boolean connectives and the $\ltlfX$ operator preserve the context. 
\end{definition}

Notice that on a $\Sigma_1$ formula no operator switches the context to $\ltlfN$, so $\finX$ of this definition produces only strong nexts and coincides with Definition~\ref{def:fin}. 

We say a trace $w$ \emph{exhibits} a limit-behavior $(M, N)$ when
\[
w \modelsinf 
\Big(
\bigwedge_{\ltlfF\ltlfG\psi \in N} \!\ltlfF\ltlfG(\psi\fgrw{M}) \wedge \bigwedge_{\ltlfG\ltlfF\psi \in M}\ltlfG\ltlfF(\psi\gfrw{N}) \Big).
\]
For the remaining section, fix one disjunct of the normal form, given by its limit-behavior $(M, N)$, and an infinite trace $w$ that exhibits $(M, N)$.

By Definition~\ref{def:body-rewrite}, the body formula $\varphi\rewM{M,N}$ of the disjunct is generated by the grammar whose atoms are the literals, $\ltlftrue$, $\ltlffalse$, and the safety leaves $\ltlfG(\psi\fgrw{M})$ with $\ltlfF\ltlfG\psi \in N$, closed under $\land$, $\lor$, $\ltlfX$, $\ltlfF$, $\ltlfU$, $\ltlfM$.

Every safety leaf is introduced by the rewrite and has a $\Pi_1$ operand $\psi\fgrw{M}$, shared with a conjunct of the limit-behavior that $w$ exhibits.

\begin{definition}[Persistence index]
\label{def:stab-index}
The trace $w$ exhibits $(M, N)$, so each $\ltlfF\ltlfG\psi \in N$ has a least position from which onward $\psi\fgrw{M}$ holds at every position of $w$.
The persistence index $T$ is the largest of these finitely many positions, and $T=0$ when $N$ is empty.\footnote{Unlike the stabilization index of~\citet{DBLP:journals/jacm/EsparzaRS24}, our trace has only to exhibit the limit-behavior of the fixed $(M, N)$ and may exhibit that of other limit-behaviors as well.}
\end{definition}

By Definition~\ref{def:stab-index}, every operand $\psi\fgrw{M}$ with $\ltlfF\ltlfG\psi \in N$ holds at every position from $T$ onward.
Every safety leaf $\ltlfG(\psi\fgrw{M})$ therefore holds from $T$ onward as well, that is, $w[i...] \modelsinf \ltlfG(\psi\fgrw{M})$ for every $i \geq T$.
Since their value is known from then on, we can \emph{project} these safety leaves out of the body formula (i.e.\ replace them with $\ltlftrue$), which yields a $\Sigma_1$ formula.
We formalize this as follows.

\begin{definition}[Projection]\label{def:proj}
For a body formula $\varphi$ of $(M, N)$, let $\proj(\varphi)$ be the formula obtained by replacing every safety leaf $\ltlfG(\psi\fgrw{M})$ of $\varphi$ with $\ltlftrue$.
\end{definition}

As this removes all safety leaves from the $\Sigma_2$ body, the resulting formula is in $\Sigma_1$.

\begin{restatable}[]{lemma}{projlemma}\label{lem:proj}
    Given a disjunct of the $\Delta_2$ normalization with limit-behavior $(M, N)$, a trace $w$ that exhibits $(M, N)$, and the persistence index $T$, for every body formula $\varphi$ and every position $i \geq T$, $w[i...] \modelsinf \varphi$ iff $w[i...] \modelsinf \proj(\varphi)$, and on every prefix $w[0..n]$ with $n \geq i$, $w[i..n] \modelsfin \finX(\varphi)$ iff $w[i..n] \modelsfin \finX(\proj(\varphi))$.
\end{restatable}

Lemma~\ref{lem:proj} states that past $T$ the same suffixes of $w$ satisfy a body formula and its projection, and the same infixes satisfy their finite reinterpretations.
From $T$ onward the body therefore behaves exactly like the $\Sigma_1$ formula $\proj(\varphi)$, and the lemmas of Section~\ref{sec:min-tight-completion} carry over to it through the projection.
In particular, Lemma~\ref{lem:2-3} extends past $T$ from $\Sigma_1$ formulas to body formulas.
Suppose $w[i..n] \modelsfin \finX(\varphi)$ for some $i \geq T$.
By Lemma~\ref{lem:proj} we may replace $\varphi$ by $\proj(\varphi)$, so $w[i..n] \modelsfin \finX(\proj(\varphi))$.
Now $\proj(\varphi) \in \Sigma_1$, so Lemma~\ref{lem:2-3} applies and gives $w[i...] \modelsinf \proj(\varphi)$.
Replacing $\proj(\varphi)$ by $\varphi$ again, by Lemma~\ref{lem:proj}, gives $w[i...] \modelsinf \varphi$.

\begin{definition}[Convergent]\label{def:convergent}
A formula $\varphi$ is convergent if there is a bound $n_0$ such that, for every $n \geq n_0$, $w[0..n] \modelsfin \finX(\varphi)$ holds if and only if $w \modelsinf \varphi$.
\end{definition}

\begin{restatable}[]{lemma}{bodystablemma}\label{lem:body-stability}
Given a disjunct of the $\Delta_2$ normalization with limit-behavior $(M, N)$ and a trace $w$ that exhibits $(M, N)$, every body formula of $(M, N)$ is convergent.
\end{restatable}

Its proof, in Appendix~\ref{appendix:sigma2}, 
is a structural induction on the body formula $\varphi\rewM{M,N}$.
The key idea is that there is a position past which every safety leaf (see Definition~\ref{def:body-rewrite}) holds, and beyond that position a body formula reduces to the $\Sigma_1$ formula in which every safety leaf is $\ltlftrue$, so a finite prefix satisfies it only when the infinite trace does (Lemma~\ref{lem:2-3}).

\begin{restatable}[]{theorem}{sigmatwotheorem}\label{theorem:sigma2}
Given a disjunct of the $\Delta_2$ normalization with limit-behavior $(M, N)$ and body formula $\varphi\rewM{M,N}$, and a trace $w$ that exhibits $(M, N)$, we obtain
    $$
    w \modelsinf \varphi\rewM{M,N} \iff w \modelsinf \exists \forall(\finX(\varphi\rewM{M,N})).
    $$    
\end{restatable}

\begin{proof}
By Lemma~\ref{lem:body-stability} the body is convergent, so there is a bound past which $w[0..n] \modelsfin \finX(\varphi[M,N])$ equals
 $w \modelsinf \varphi[M,N]$.

($\Rightarrow$) Suppose $w \modelsinf \varphi[M,N]$ holds. Then for every $n$ past the bound, $w[0..n] \modelsfin \finX(\varphi[M,N])$ is true, because it equals $w \modelsinf \varphi[M,N]$. So all but finitely many prefixes satisfy $\finX(\varphi[M,N])$, which is $\exists\forall(\finX(\varphi[M,N]))$.

($\Leftarrow$) Suppose $\exists\forall(\finX(\varphi[M,N]))$, so all but finitely many prefixes satisfy $\finX(\varphi[M,N])$. Since only finitely many fail, some satisfying prefix $w[0..n]$ has $n$ past the bound. There $w[0..n] \modelsfin \finX(\varphi[M,N])$ holds and equals $w \modelsinf \varphi[M,N]$, so $w \modelsinf \varphi[M,N]$ holds too.
\end{proof}

\section{Putting It All Together}\label{putting-together}
Starting from the $\Delta_2$ normal form of Theorem~\ref{thm:normalisation}, we translate the three kinds of conjunct that form each disjunct of the normalization.
Theorem~\ref{theorem:AE} translates the conjuncts $\ltlfG\ltlfF(\sigma)$ with $\sigma \in \Sigma_1$, Theorem~\ref{theorem:EA} the conjuncts $\ltlfF\ltlfG(\pi)$ with $\pi \in \Pi_1$, and Theorem~\ref{theorem:sigma2} the $\Sigma_2$ bodies.
Theorems~\ref{theorem:AE} and~\ref{theorem:EA} hold for every trace, while Theorem~\ref{theorem:sigma2} needs the trace to exhibit the disjunct's limit-behavior, that is, to satisfy its limit-behavior conjuncts.
Since the disjunct itself contains these conjuncts, this is no restriction.
When they hold, Theorem~\ref{theorem:sigma2} translates the body.
When one fails, the disjunct is false on both sides of the translation.
We can now state the translation in full and prove it disjunct by disjunct.

\begin{theorem}\label{thm:main}
For every \ltl formula $\varphi$ and every $w \in \Sigma^\omega$,
{\small
\begin{multline*}
w \modelsinf \varphi \iff{}
w \modelsinf \!\!\!\bigvee_{\substack{M \subseteq \BGF \\ N \subseteq \BFG}}\!\!\!
\Big(\exists\forall\big(\finX(\varphi\rewM{M,N})\big)
~\wedge \\
\bigwedge_{\ltlfF\ltlfG\psi \in N}\!\!\! \!\!\exists\forall\big(\ltlfG(\neg\tight(\neg(\psi\fgrw{M})))\big)
~\wedge\!\!\!\!\!
\bigwedge_{\ltlfG\ltlfF\psi \in M}\!\!\!\!\! \forall\exists\big(\ltlfF(\tight(\psi\gfrw{N}))\big)\!\Big),
\end{multline*}}
where $\rewM{M,N}$, $\fgrw{M}$, and $\gfrw{N}$ are the rewrites of the \normalisation (Definitions~\ref{def:fg-rewrite}--\ref{def:body-rewrite}), $\tight$ is the tight completion (Definition~\ref{def:tight}), and $\finX$ the extended finite reinterpretation (Definition~\ref{def:fin-c}).
\end{theorem}

\begin{proof}
Fix $w \in \Sigma^\omega$.
By Theorem~\ref{thm:normalisation}, $w \modelsinf \varphi$ if and only if $w$ satisfies a disjunct
$\varphi\rewM{M,N} \wedge \bigwedge_{\ltlfF\ltlfG\psi \in N} \ltlfF\ltlfG(\psi\fgrw{M}) \wedge \bigwedge_{\ltlfG\ltlfF\psi \in M} \ltlfG\ltlfF(\psi\gfrw{N})$
of its $\Delta_2$ normal form, specifically the one that matches the limit behavior $(M, N)$ of $w$.
It therefore suffices to prove the equivalence for every $(M, N)$.

We first translate the limit-behavior conjuncts.
Every operand $\psi\fgrw{M}$ with $\ltlfF\ltlfG\psi \in N$ is in $\Pi_1$ by Definition~\ref{def:fg-rewrite}, so by Theorem~\ref{theorem:EA}, $w \modelsinf \ltlfF\ltlfG(\psi\fgrw{M})$ if and only if $w \modelsinf \exists\forall(\ltlfG(\neg\tight(\neg(\psi\fgrw{M}))))$.
Every operand $\psi\gfrw{N}$ with $\ltlfG\ltlfF\psi \in M$ is in $\Sigma_1$ by Definition~\ref{def:gf-rewrite}, so by Theorem~\ref{theorem:AE}, $w \modelsinf \ltlfG\ltlfF(\psi\gfrw{N})$ if and only if $w \modelsinf \forall\exists(\ltlfF(\tight(\psi\gfrw{N})))$.
Conjoining these equivalences, $w$ satisfies the limit-behavior conjuncts of the disjunct if and only if $w$ satisfies their translations.

Now we distinguish whether $w$ satisfies the limit-behavior conjuncts.
If it does not, then it does not satisfy their translations either, so each side of the claimed equivalence contains a false conjunct and both sides are false.
If it does, then it satisfies their translations, so on each side every conjunct but the first is true, and only $w \modelsinf \varphi\rewM{M,N} \iff w \modelsinf \exists\forall(\finX(\varphi\rewM{M,N}))$ remains.
The case assumption states precisely that $w$ exhibits $(M, N)$ (Section~\ref{sec:sigma2}), so Theorem~\ref{theorem:sigma2} applies and yields this remaining equivalence.
\end{proof}

The translated formula inherits the length bound $2^{O(|\varphi|)}$ of the $\Delta_2$ normal form (Theorem~\ref{thm:normalisation}), as every added step is linear. Moreover, the exponent comes only from the number of disjuncts, each leaf being of polynomial length, so the automata constructions of \citet{DBLP:conf/ijcai/AminofGRV25}, applied leaf-wise, yield a deterministic automaton of doubly exponential size, as in Section~\ref{sec:normalisation}.

\section{Conclusion}\label{sec:conclusion}

In this paper we showed that \ltl formulas, once in normal form, can be translated into \ltlfplus in linear time.
Given an \ltl formula $\varphi$, its normal form is worst-case of size $2^{O(|\varphi|)}$.
This exponential blow-up, however, does not reflect in a higher complexity of translating LTL to deterministic automata, which remains worst-case in 2EXPTIME (the problem is 2EXPTIME-complete).
In fact, the simplicity of handling the single components makes it the best option for solving \ltl synthesis and is largely adopted by the state-of-the art synthesis tools \citep{DBLP:conf/tacas/KretinskyMPZ25,DBLP:journals/corr/abs-2604-24102,DBLP:conf/cav/MeyerSL18}.
Our translation makes the DFA-based technology developed for \ltlf available to \ltl synthesis.
In particular, all intermediate DFAs can be minimized, which can drastically speed up synthesis in practice.
Furthermore, beyond synthesis, our translation affects all areas which require \ltl to be translated to deterministic (or limit deterministic)
automata.
This includes Markov decision processes \citep{DBLP:conf/ijcai/Giacomo0SWY25}, stochastic planning \cite{DBLP:conf/aaai/WeinhuberGLST26}, %
and reinforcement learning \citep{DBLP:conf/tacas/HahnPSS0W20, DBLP:conf/iclr/JackermeierA25, IJCAI26}.
We thereby establish LTLf+ as a practical intermediate language for arbitrary LTL specifications in a variety of contexts.

We leave the practical realization of this route to future work.
As a first step, we have already implemented the full translation on top of the
\spot library \citep{DBLP:conf/cav/Duret-LutzRCRAS22}.
The implementation goes beyond the plain construction and includes several optimizations.
In particular, for the widely used obligation fragment~\cite{DBLP:journals/corr/abs-2605-12372} it skips both the
normalization and the tight completion and directly returns obligation formulas
of \ltlfplus.

\section*{Acknowledgments}
This work was supported in part by the UKRI Erlangen AI Hub on Mathematical and Computational Foundations of AI (No. EP/Y028872/1) and by the EU Horizon Europe project euroFMX (No. 101299128).

\bibliography{aaai2027}

@inproceedings{DBLP:conf/ijcai/AminofGRV25,
  author       = {Benjamin Aminof and
                  Giuseppe {De Giacomo} and
                  Sasha Rubin and
                  Moshe Y. Vardi},
  title        = {{LTLf+} and {PPLTL+:} Extending {LTLf} and {PPLTL} to Infinite Traces},
  booktitle    = {{IJCAI}},
  pages        = {8447--8455},
  publisher    = {ijcai.org},
  year         = {2025}
}

@inproceedings{DBLP:conf/ijcai/Giacomo0SWY25,
  author       = {Giuseppe {De Giacomo} and
                  Yong Li and
                  Sven Schewe and
                  Christoph Weinhuber and
                  Pian Yu},
  title        = {Solving {MDPs} with {LTLf+} and {PPLTL+} Temporal Objectives},
  booktitle    = {{IJCAI}},
  pages        = {8491--8499},
  publisher    = {ijcai.org},
  year         = {2025}
}

@inproceedings{DBLP:conf/podc/MannaP89,
  author       = {Zohar Manna and
                  Amir Pnueli},
  title        = {A Hierarchy of Temporal Properties},
  booktitle    = {{PODC}},
  pages        = {377--410},
  publisher    = {{ACM}},
  year         = {1990}
}

@inproceedings{DBLP:conf/icalp/ChangMP92,
  author       = {Edward Y. Chang and
                  Zohar Manna and
                  Amir Pnueli},
  title        = {Characterization of Temporal Property Classes},
  booktitle    = {{ICALP}},
  series       = {Lecture Notes in Computer Science},
  pages        = {474--486},
  publisher    = {Springer},
  year         = {1992}
}

@inproceedings{DBLP:conf/cav/Duret-LutzRCRAS22,
  author       = {Alexandre Duret{-}Lutz and
                  Etienne Renault and
                  Maximilien Colange and
                  Florian Renkin and
                  Alexandre Gbaguidi Aisse and
                  Philipp Schlehuber{-}Caissier and
                  Thomas Medioni and
                  Antoine Martin and
                  J{\'{e}}r{\^{o}}me Dubois and
                  Cl{\'{e}}ment Gillard and
                  Henrich Lauko},
  title        = {From Spot 2.0 to Spot 2.10: What's New?},
  booktitle    = {Computer Aided Verification - 34th International Conference, {CAV}
                  2022},
  series       = {Lecture Notes in Computer Science},
  pages        = {174--187},
  publisher    = {Springer},
  year         = {2022}
}

@inproceedings{DBLP:conf/ijcai/GiacomoV13,
  author       = {Giuseppe {De Giacomo} and
                  Moshe Y. Vardi},
  title        = {Linear Temporal Logic and Linear Dynamic Logic on Finite Traces},
  booktitle    = {{IJCAI}},
  pages        = {854--860},
  publisher    = {{IJCAI/AAAI}},
  year         = {2013}
}

@inproceedings{DBLP:conf/focs/Pnueli77,
  author       = {Amir Pnueli},
  title        = {The Temporal Logic of Programs},
  booktitle    = {{FOCS}},
  pages        = {46--57},
  publisher    = {{IEEE} Computer Society},
  year         = {1977}
}

@article{DBLP:journals/jacm/EsparzaRS24,
  author       = {Javier Esparza and
                  Rub{\'{e}}n Rubio and
                  Salomon Sickert},
  title        = {Efficient Normalization of Linear Temporal Logic},
  journal      = {J. {ACM}},
  volume       = {71},
  number       = {2},
  pages        = {16:1--16:42},
  year         = {2024}
}

@article{DBLP:journals/jacm/EsparzaKS20,
  author       = {Javier Esparza and
                  Jan Kret{\'{\i}}nsk{\'{y}} and
                  Salomon Sickert},
  title        = {A Unified Translation of Linear Temporal Logic to {\(\omega\)}-Automata},
  journal      = {J. {ACM}},
  volume       = {67},
  number       = {6},
  pages        = {33:1--33:61},
  year         = {2020}
}

@book{DBLP:books/daglib/0020348,
  author       = {Christel Baier and
                  Joost{-}Pieter Katoen},
  title        = {Principles of model checking},
  publisher    = {{MIT} Press},
  year         = {2008}
}

@inproceedings{DBLP:conf/popl/PnueliR89,
  author       = {Amir Pnueli and
                  Roni Rosner},
  title        = {On the Synthesis of a Reactive Module},
  booktitle    = {{POPL}},
  pages        = {179--190},
  publisher    = {{ACM} Press},
  year         = {1989}
}

@article{DBLP:journals/deds/EhlersLTV17,
  author       = {R{\"{u}}diger Ehlers and
                  St{\'{e}}phane Lafortune and
                  Stavros Tripakis and
                  Moshe Y. Vardi},
  title        = {Supervisory control and reactive synthesis: a comparative introduction},
  journal      = {Discret. Event Dyn. Syst.},
  volume       = {27},
  number       = {2},
  pages        = {209--260},
  year         = {2017}
}

@article{DBLP:journals/amai/BacchusK98,
  author       = {Fahiem Bacchus and
                  Froduald Kabanza},
  title        = {Planning for Temporally Extended Goals},
  journal      = {Ann. Math. Artif. Intell.},
  volume       = {22},
  number       = {1-2},
  pages        = {5--27},
  year         = {1998}
}

@inproceedings{DBLP:conf/ijcai/GiacomoR18,
  author       = {Giuseppe {De Giacomo} and
                  Sasha Rubin},
  title        = {Automata-Theoretic Foundations of {FOND} Planning for LTLf and LDLf
                  Goals},
  booktitle    = {{IJCAI}},
  pages        = {4729--4735},
  publisher    = {ijcai.org},
  year         = {2018}
}

@inproceedings{DBLP:conf/aips/CamachoBM19,
  author       = {Alberto Camacho and
                  Meghyn Bienvenu and
                  Sheila A. McIlraith},
  title        = {Towards a Unified View of {AI} Planning and Reactive Synthesis},
  booktitle    = {{ICAPS}},
  pages        = {58--67},
  publisher    = {{AAAI} Press},
  year         = {2019}
}

@article{DBLP:journals/automatica/FainekosGKP09,
  author       = {Georgios E. Fainekos and
                  Antoine Girard and
                  Hadas Kress{-}Gazit and
                  George J. Pappas},
  title        = {Temporal logic motion planning for dynamic robots},
  journal      = {Autom.},
  volume       = {45},
  number       = {2},
  pages        = {343--352},
  year         = {2009}
}

@inproceedings{DBLP:conf/bpm/MaggiMWA11,
  author       = {Fabrizio Maria Maggi and
                  Marco Montali and
                  Michael Westergaard and
                  Wil M. P. van der Aalst},
  title        = {Monitoring Business Constraints with Linear Temporal Logic: An Approach
                  Based on Colored Automata},
  booktitle    = {{BPM}},
  series       = {Lecture Notes in Computer Science},
  pages        = {132--147},
  publisher    = {Springer},
  year         = {2011}
}

@inproceedings{DBLP:conf/stacs/Vardi07,
  author       = {Moshe Y. Vardi},
  title        = {The B{\"{u}}chi Complementation Saga},
  booktitle    = {{STACS}},
  series       = {Lecture Notes in Computer Science},
  pages        = {12--22},
  publisher    = {Springer},
  year         = {2007}
}

@inproceedings{DBLP:conf/ijcai/GiacomoV15,
  author       = {Giuseppe {De Giacomo} and
                  Moshe Y. Vardi},
  title        = {Synthesis for {LTL} and {LDL} on Finite Traces},
  booktitle    = {{IJCAI}},
  pages        = {1558--1564},
  publisher    = {{AAAI} Press},
  year         = {2015}
}

@inproceedings{DBLP:conf/ijcai/ZhuTLPV17,
  author       = {Shufang Zhu and
                  Lucas M. Tabajara and
                  Jianwen Li and
                  Geguang Pu and
                  Moshe Y. Vardi},
  title        = {Symbolic LTLf Synthesis},
  booktitle    = {{IJCAI}},
  pages        = {1362--1369},
  publisher    = {ijcai.org},
  year         = {2017}
}

@inproceedings{DBLP:conf/aaai/BansalLTV20,
  author       = {Suguman Bansal and
                  Yong Li and
                  Lucas M. Tabajara and
                  Moshe Y. Vardi},
  title        = {Hybrid Compositional Reasoning for Reactive Synthesis from Finite-Horizon
                  Specifications},
  booktitle    = {{AAAI}},
  pages        = {9766--9774},
  publisher    = {{AAAI} Press},
  year         = {2020}
}

@inproceedings{DBLP:conf/aips/GiacomoF21,
  author       = {Giuseppe {De Giacomo} and
                  Marco Favorito},
  title        = {Compositional Approach to Translate LTLf/LDLf into Deterministic Finite
                  Automata},
  booktitle    = {{ICAPS}},
  pages        = {122--130},
  publisher    = {{AAAI} Press},
  year         = {2021}
}

@inproceedings{DBLP:conf/wia/DuretLutzZPGV25,
  author       = {Alexandre Duret{-}Lutz and
                  Shufang Zhu and
                  Nir Piterman and
                  Giuseppe {De Giacomo} and
                  Moshe Y. Vardi},
  title        = {Engineering an LTL\({}_{\mbox{f}}\) Synthesis Tool},
  booktitle    = {{CIAA}},
  series       = {Lecture Notes in Computer Science},
  pages        = {129--147},
  publisher    = {Springer},
  year         = {2025}
}

@article{DBLP:journals/scp/EmersonL87,
  author       = {E. Allen Emerson and
                  Chin{-}Laung Lei},
  title        = {Modalities for Model Checking: Branching Time Logic Strikes Back},
  journal      = {Sci. Comput. Program.},
  volume       = {8},
  number       = {3},
  pages        = {275--306},
  year         = {1987}
}

@inproceedings{DBLP:conf/kr/00010PWGP25,
  author       = {Daniel Hausmann and
                  Shufang Zhu and
                  Gianmarco Parretti and
                  Christoph Weinhuber and
                  Giuseppe {De Giacomo} and
                  Nir Piterman},
  title        = {Emerson-Lei and Manna-Pnueli Games for {LTLf+} and {PPLTL+} Synthesis},
  booktitle    = {{KR}},
  year         = {2025}
}

@article{DBLP:journals/fmsd/KupfermanV01,
  author       = {Orna Kupferman and
                  Moshe Y. Vardi},
  title        = {Model Checking of Safety Properties},
  journal      = {Formal Methods Syst. Des.},
  volume       = {19},
  number       = {3},
  pages        = {291--314},
  year         = {2001}
}

@article{DBLP:journals/lmcs/CimattiGGMT23,
  author       = {Alessandro Cimatti and
                  Luca Geatti and
                  Nicola Gigante and
                  Angelo Montanari and
                  Stefano Tonetta},
  title        = {A first-order logic characterization of safety and co-safety languages},
  journal      = {Log. Methods Comput. Sci.},
  volume       = {19},
  number       = {3},
  year         = {2023}
}

@inproceedings{DBLP:conf/aaai/WeinhuberGLST26,
  author       = {Christoph Weinhuber and
                  Giuseppe {De Giacomo} and
                  Yong Li and
                  Sven Schewe and
                  Qiyi Tang},
  title        = {Good-for-MDP State Reduction for Stochastic {LTL} Planning},
  booktitle    = {{AAAI}},
  pages        = {36457--36465},
  publisher    = {{AAAI} Press},
  year         = {2026}
}

@inproceedings{DBLP:conf/mochart/KupfermanR10,
  author       = {Orna Kupferman and
                  Adin Rosenberg},
  title        = {The Blowup in Translating {LTL} to Deterministic Automata},
  booktitle    = {MoChArt},
  series       = {Lecture Notes in Computer Science},
  volume       = {6572},
  pages        = {85--94},
  publisher    = {Springer},
  year         = {2010}
}

@article{DBLP:journals/corr/abs-2604-18532,
  author       = {Giuseppe {De Giacomo} and
                  Christian Hagemeier and
                  Daniel Hausmann and
                  Nir Piterman},
  title        = {Symbolic Synthesis for LTLf+ Obligations},
  journal      = {CoRR},
  volume       = {abs/2604.18532},
  year         = {2026}
}

@inproceedings{DBLP:conf/tacas/HahnPSS0W20,
  author       = {Ernst Moritz Hahn and
                  Mateo Perez and
                  Sven Schewe and
                  Fabio Somenzi and
                  Ashutosh Trivedi and
                  Dominik Wojtczak},
  title        = {Good-for-MDPs Automata for Probabilistic Analysis and Reinforcement
                  Learning},
  booktitle    = {{TACAS} {(1)}},
  series       = {Lecture Notes in Computer Science},
  volume       = {12078},
  pages        = {306--323},
  publisher    = {Springer},
  year         = {2020}
}

@inproceedings{DBLP:conf/lop/LichtensteinPZ85,
  author       = {Orna Lichtenstein and
                  Amir Pnueli and
                  Lenore D. Zuck},
  title        = {The Glory of the Past},
  booktitle    = {Logic of Programs},
  series       = {Lecture Notes in Computer Science},
  volume       = {193},
  pages        = {196--218},
  publisher    = {Springer},
  year         = {1985}
}

@incollection{DBLP:reference/mc/PitermanP18,
  author       = {Nir Piterman and
                  Amir Pnueli},
  title        = {Temporal Logic and Fair Discrete Systems},
  booktitle    = {Handbook of Model Checking},
  pages        = {27--73},
  publisher    = {Springer},
  year         = {2018}
}

@inproceedings{DBLP:journals/corr/abs-1709-02102,
  author       = {David M{\"{u}}ller and
                  Salomon Sickert},
  title        = {{LTL} to Deterministic Emerson-Lei Automata},
  booktitle    = {GandALF},
  series       = {{EPTCS}},
  volume       = {256},
  pages        = {180--194},
  year         = {2017}
}

@inproceedings{IJCAI26,
  author       = {Alessandro Abate and Giuseppe {De Giacomo} and Mathias Jackermeier and Jan Kret{\'{\i}}nsk{\'{y}} and Maximilian Prokop and Christoph Weinhuber},
  title        = {Semantically Labelled Automata for Multi-Task Reinforcement Learning with {LTL} Instructions},
  booktitle    = {{IJCAI}},
  publisher    = {ijcai.org},
  year         = {2026}
}

@inproceedings{DBLP:conf/ijcai/GiacomoSFR20,
  author       = {Giuseppe {De Giacomo} and
                  Antonio Di Stasio and
                  Francesco Fuggitti and
                  Sasha Rubin},
  title        = {Pure-Past Linear Temporal and Dynamic Logic on Finite Traces},
  booktitle    = {{IJCAI}},
  pages        = {4959--4965},
  publisher    = {ijcai.org},
  year         = {2020}
}

@inproceedings{DBLP:conf/cav/MeyerSL18,
  author       = {Philipp J. Meyer and
                  Salomon Sickert and
                  Michael Luttenberger},
  title        = {Strix: Explicit Reactive Synthesis Strikes Back!},
  booktitle    = {{CAV} {(1)}},
  series       = {Lecture Notes in Computer Science},
  volume       = {10981},
  pages        = {578--586},
  publisher    = {Springer},
  year         = {2018}
}

@inproceedings{DBLP:conf/tacas/KretinskyMPZ25,
  author       = {Jan Kret{\'{\i}}nsk{\'{y}} and
                  Tobias Meggendorfer and
                  Maximilian Prokop and
                  Ashkan Zarkhah},
  title        = {SemML: Enhancing Automata-Theoretic {LTL} Synthesis with Machine Learning},
  booktitle    = {{TACAS} {(1)}},
  series       = {Lecture Notes in Computer Science},
  volume       = {15696},
  pages        = {233--253},
  publisher    = {Springer},
  year         = {2025}
}

@article{DBLP:journals/corr/abs-2604-24102,
  author       = {Jan Kret{\'{\i}}nsk{\'{y}} and
                  Tobias Meggendorfer and
                  Maximilian Prokop},
  title        = {SemML 2.0: Synthesizing Controllers for {LTL}},
  journal      = {CoRR},
  volume       = {abs/2604.24102},
  year         = {2026}
}

@ARTICLE{PSL,
  journal={IEC 62531:2012(E) (IEEE Std 1850-2010)}, 
  title={{{IEC}} 62531:2012(E) (IEEE Std 1850-2010): Standard for Property Specification Language (PSL)}, 
  year={2012},
  volume={},
  number={},
  nopages={1-184},
  doi={10.1109/IEEESTD.2012.6228486}
}

@ARTICLE{SVA,
  journal={{IEEE} Std 1800-2017 (Revision of IEEE Std 1800-2012)}, 
  title={{IEEE} Standard for SystemVerilog--Unified Hardware Design, Specification, and Verification Language}, 
  year={2018},
  volume={},
  number={},
  nopages={1-1315},
  doi={10.1109/IEEESTD.2018.8299595}}

@inproceedings{DBLP:conf/iclr/JackermeierA25,
  author       = {Mathias Jackermeier and
                  Alessandro Abate},
  title        = {DeepLTL: Learning to Efficiently Satisfy Complex {LTL} Specifications
                  for Multi-Task {RL}},
  booktitle    = {{ICLR}},
  publisher    = {OpenReview.net},
  year         = {2025}
}

@article{DBLP:journals/corr/abs-2605-12372,
  author       = {Alexandre Duret{-}Lutz and
                  Giuseppe De Giacomo and
                  Marcin Jurdzinski and
                  Nir Piterman and
                  Moshe Y. Vardi and
                  Shufang Zhu},
  title        = {Fast Obligation Translation and Synthesis},
  journal      = {CoRR},
  volume       = {abs/2605.12372},
  year         = {2026}
}

\clearpage
\appendix
\section{Related Work}\label{sec:related}

\paragraph{Normal forms for the safety-progress hierarchy.}
The hierarchy originates with \citet{DBLP:conf/lop/LichtensteinPZ85}, was described in detail and named by \citet{DBLP:conf/podc/MannaP89}, and is surveyed by \citet{DBLP:reference/mc/PitermanP18}.
In particular, every formula of \ltl with past operators is equivalent to a formula $\bigwedge_i \big(\ltlfG\ltlfF(\varphi_i) \lor \ltlfF\ltlfG(\psi_i)\big)$ where $\varphi_i$ and $\psi_i$ contain only past operators \citep{DBLP:conf/lop/LichtensteinPZ85,DBLP:conf/podc/MannaP89}.
\citet{DBLP:conf/icalp/ChangMP92} carried the characterization over to future-only \ltl and stated the corresponding $\Delta_2$ normalization theorem.
The classical procedures behind these results incur a non-elementary blow-up, translating formulas into counter-free automata and star-free expressions and back.
The elementary alternatives construct deterministic $\omega$-automata as an intermediate step and are therefore at least doubly exponential.
We refer to \citet{DBLP:journals/jacm/EsparzaRS24} for a detailed description.
The \normalisation of \citet{DBLP:journals/jacm/EsparzaRS24} that our translation builds on is instead direct and purely syntactic, and incurs only a single-exponential blow-up.
Together with the unified \ltl-to-$\omega$-automata translation of \citet{DBLP:journals/jacm/EsparzaKS20}, it has renewed interest in the hierarchy \citep{DBLP:conf/ijcai/AminofGRV25}.

\paragraph{\ltl to deterministic automata.}
Translating \ltl to deterministic $\omega$-automata is doubly exponential, and this is tight \citep{DBLP:conf/mochart/KupfermanR10}.
Moreover, the classical route through \buchi automata and Safra-style determinisation is notoriously difficult \citep{DBLP:conf/stacs/Vardi07}.
Deterministic Emerson-Lei automata \citep{DBLP:journals/scp/EmersonL87} carry a Muller acceptance condition that is expressed symbolically as a Boolean formula, which makes it possible to shift complexity from the state space to the acceptance condition.
In particular, \citet{DBLP:journals/corr/abs-1709-02102} translate the (co-)safety and fairness fragments of \ltl directly, delegate subformulas outside these fragments to external translators, and compose the parts by a product.

\paragraph{Finite-trace logics and \ltlfplus.}
\ltlf interprets temporal specifications over finite traces, where the automaton counterpart is the DFA \citep{DBLP:conf/ijcai/GiacomoV13,DBLP:conf/ijcai/GiacomoV15}, and supports a family of scalable synthesis tools \citep{DBLP:conf/ijcai/ZhuTLPV17,DBLP:conf/aaai/BansalLTV20,DBLP:conf/aips/GiacomoF21,DBLP:conf/wia/DuretLutzZPGV25}.
Its pure-past variant PPLTL sees the trace backwards and admits DFAs of single-exponential size \citep{DBLP:conf/ijcai/GiacomoSFR20}.
\ltlfplus and PPLTL+ lift this machinery to infinite traces through prefix quantification \citep{DBLP:conf/ijcai/AminofGRV25}, with synthesis through Emerson-Lei games and Manna-Pnueli games \citep{DBLP:conf/kr/00010PWGP25}, symbolic algorithms for the obligation fragment \citep{DBLP:journals/corr/abs-2604-18532}, and \ltlfplus objectives in MDPs \citep{DBLP:conf/ijcai/Giacomo0SWY25,DBLP:conf/aaai/WeinhuberGLST26}.
\citet{DBLP:conf/ijcai/AminofGRV25} prove that \ltlfplus captures all of \ltl but do not provide a translation. 
Our translation supplies the missing step and connects the normalization with the prefix quantifiers.
Our tight completion builds on classical prefix characterizations of safety and co-safety \citep{DBLP:journals/fmsd/KupfermanV01,DBLP:journals/lmcs/CimattiGGMT23}.

\section{Supplementary Material for Section~\ref{sec:normalisation}}\label{appendix:normalisation}
The $\Sigma_2$ body rewrite of Definition~\ref{def:body-rewrite} adapts \citet[Definition~4.19]{DBLP:journals/jacm/EsparzaRS24}.
A kept $\nu$-operator is rewritten as there, and a dropped one as in the corresponding clauses of \citet[Definition~28]{DBLP:journals/jacm/EsparzaKS20}.
The one difference is that we keep the safety leaf \emph{conditionally}, dropping it to $\ltlffalse$ unless $\ltlfF\ltlfG\psi_1 \in N$.

We write $\varphi\rewMN{M}$ for the rewrite of \citet[Definition~4.19]{DBLP:journals/jacm/EsparzaRS24}, which keeps every safety leaf unconditionally:
\begin{align*}
(\ltlfG\,\psi_1)\rewMN{M} &= \psi_1\rewMN{M}\, \ltlfU\, \ltlfG(\psi_1\fgrw{M}), \\
(\psi_1 \ltlfW \psi_2)\rewMN{M} &= \psi_1\rewMN{M}\, \ltlfU\, \bigl(\psi_2\rewMN{M} \lor \ltlfG(\psi_1\fgrw{M})\bigr), \\
(\psi_2 \ltlfR \psi_1)\rewMN{M} &= \bigl(\psi_2\rewMN{M} \lor \ltlfG(\psi_1\fgrw{M})\bigr)\, \ltlfM\, \psi_1\rewMN{M},
\end{align*}
and otherwise agrees with Definition~\ref{def:body-rewrite}.
\citet[Theorem~4.22]{DBLP:journals/jacm/EsparzaRS24} state that every \ltl formula $\varphi$ is equivalent to the normal form of Theorem~\ref{thm:normalisation} with bodies $\varphi\rewMN{M}$ in place of $\varphi\rewM{M,N}$, and otherwise identical.

\begin{lemma}\label{lem:variant-strengthening}
For every $(M, N)$, every subformula $\varphi'$ of $\varphi$, every $w \in \Sigma^\omega$, and every position $i$, we have
$w, i \modelsinf \varphi'\rewM{M,N} \implies w, i \modelsinf \varphi'\rewMN{M}$.
\end{lemma}

\begin{proof}
Structural induction on $\varphi'$.
\begin{itemize}
\item Literals, $\ltlftrue$, and $\ltlffalse$ are unchanged by both rewrites.
\item Both rewrites distribute over $\land$, $\lor$, $\ltlfX$, $\ltlfF$, $\ltlfU$, $\ltlfM$, so these cases follow by applying the induction hypothesis.
\item A kept $\ltlfG$, $\ltlfW$, or $\ltlfR$ has the same clause with the same safety leaf in both rewrites, and the claim again follows by applying the induction hypothesis.
\item A dropped $\ltlfG\psi$ is $\ltlffalse$, which implies every formula.
\item A dropped $\psi_1 \ltlfW \psi_2$ is $\psi_1\rewM{M,N} \ltlfU \psi_2\rewM{M,N}$. The induction hypothesis on both operands gives $\psi_1\rewMN{M} \ltlfU \psi_2\rewMN{M}$, and weakening the second operand by the disjunct $\ltlfG(\psi_1\fgrw{M})$ gives $(\psi_1 \ltlfW \psi_2)\rewMN{M}$.
\item A dropped $\psi_2 \ltlfR \psi_1$ is $\psi_2\rewM{M,N} \ltlfM \psi_1\rewM{M,N}$, and as before the induction hypothesis on both operands and weakening the first operand by the disjunct $\ltlfG(\psi_1\fgrw{M})$ give $(\psi_2 \ltlfR \psi_1)\rewMN{M}$.
\qedhere
\end{itemize}
\end{proof}

\begin{lemma}[Compare with {\citealp[Proposition~4.16]{DBLP:journals/jacm/EsparzaRS24}}]\label{lem:mu-monotone}
Let $N \subseteq N' \subseteq \BFG$. For every formula $\psi$, every $w \in \Sigma^\omega$, and every position $i$, we have
$w, i \modelsinf \psi\gfrw{N} \implies w, i \modelsinf \psi\gfrw{N'}$.
\end{lemma}

\begin{proof}
Structural induction on $\psi$.
\begin{itemize}
\item Literals, $\ltlftrue$, and $\ltlffalse$ are unchanged, and the rewrite distributes over $\lor$, $\land$, $\ltlfX$, $\ltlfF$, $\ltlfU$, $\ltlfM$, so these cases follow by applying the induction hypothesis to the operands.
\item For $\psi_1 \ltlfW \psi_2$ with $\ltlfF\ltlfG\psi_1 \in N$, both rewrites yield $\ltlftrue$.
\item For $\psi_1 \ltlfW \psi_2$ with $\ltlfF\ltlfG\psi_1 \in N' \setminus N$, the left-hand side is a $\ltlfU$-formula and the right-hand side is $\ltlftrue$.
\item For $\psi_1 \ltlfW \psi_2$ with $\ltlfF\ltlfG\psi_1 \notin N'$, both rewrites yield $\ltlfU$-formulas over operands related by the induction hypothesis.
\item The $\ltlfR$ case is analogous with the condition on $\psi_2$, and in the $\ltlfG$ case both sides are constants with the left implying the right.
\qedhere
\end{itemize}
\end{proof}

\begin{lemma}\label{lem:variant-agreement}
Let $(M, N)$ be given and let $w \in \Sigma^\omega$ be a trace with $w \not\modelsinf \ltlfF\ltlfG(\psi\fgrw{M})$ for every $\ltlfF\ltlfG\psi \in \BFG \setminus N$.
Then for every subformula $\varphi'$ of $\varphi$ and every position $i$, we have
$w, i \modelsinf \varphi'\rewM{M,N} \iff w, i \modelsinf \varphi'\rewMN{M}$.
\end{lemma}

\begin{proof}
For every $\ltlfF\ltlfG\psi \in \BFG \setminus N$, no position of $w$ satisfies $\ltlfG(\psi\fgrw{M})$, since otherwise $w \modelsinf \ltlfF\ltlfG(\psi\fgrw{M})$ would hold.
By Definition~\ref{def:basis} the $\ltlfF\ltlfG$-formula of every $\nu$-operator of $\varphi'$ lies in $\BFG$, so every dropped operator falls under this observation.
We proceed by structural induction on $\varphi'$.
\begin{itemize}
\item The unchanged and distributing cases transfer the equivalence from the operands, and a kept $\nu$-operator has the same clause with the same safety leaf on both sides.
\item For a dropped $\ltlfG\psi$, the clause $\psi\rewMN{M} \ltlfU \ltlfG(\psi\fgrw{M})$ requires a position satisfying the leaf, so both sides are false at every position.
\item For a dropped $\psi_1 \ltlfW \psi_2$, the disjunct $\ltlfG(\psi_1\fgrw{M})$ is false at every position of $w$, so on $w$ the clause is equivalent to $\psi_1\rewMN{M} \ltlfU \psi_2\rewMN{M}$, and the induction hypothesis transfers both directions.
\item For a dropped $\psi_2 \ltlfR \psi_1$, the disjunct $\ltlfG(\psi_1\fgrw{M})$ in the first operand is false at every position of $w$, so on $w$ the clause is equivalent to $\psi_2\rewMN{M} \ltlfM \psi_1\rewMN{M}$, and the induction hypothesis transfers both directions.
\qedhere
\end{itemize}
\end{proof}

\normalisationthm*

For every \ltl formula $\varphi$ and every $w \in \Sigma^\omega$, $w$ satisfies the normal form of Theorem~\ref{thm:normalisation} if and only if $w$ satisfies the normal form of \citet[Theorem~4.22]{DBLP:journals/jacm/EsparzaRS24}.
In particular, the equivalence stated in Theorem~\ref{thm:normalisation} holds.

\begin{proof}
The limit-behavior conjuncts of the two normal forms coincide, and the disjuncts differ only in their bodies.

($\Rightarrow$) If $w$ satisfies the disjunct of $(M, N)$ in Theorem~\ref{thm:normalisation}, then $w \modelsinf \varphi\rewMN{M}$ by Lemma~\ref{lem:variant-strengthening} at position $0$, so $w$ satisfies the disjunct of the same $(M, N)$ in the normal form of \citet[Theorem~4.22]{DBLP:journals/jacm/EsparzaRS24}.

($\Leftarrow$) Suppose $w$ satisfies the disjunct of $(M, N)$ in the normal form of \citet[Theorem~4.22]{DBLP:journals/jacm/EsparzaRS24}, and define $N^* = N \cup \{\, \ltlfF\ltlfG\psi \in \BFG \mid w \modelsinf \ltlfF\ltlfG(\psi\fgrw{M}) \,\}$.
The persistence blocks of $(M, N^*)$ hold, for $\ltlfF\ltlfG\psi \in N$ by assumption and otherwise by the definition of $N^*$.
The recurrence blocks hold by applying $\ltlfG\ltlfF$ to the position-wise implication of Lemma~\ref{lem:mu-monotone}.
Every $\ltlfF\ltlfG\psi \in \BFG \setminus N^*$ satisfies $w \not\modelsinf \ltlfF\ltlfG(\psi\fgrw{M})$, so Lemma~\ref{lem:variant-agreement} applies to $(M, N^*)$ and yields $w \modelsinf \varphi\rewM{M,N^*}$ from $w \modelsinf \varphi\rewMN{M}$.
Hence $w$ satisfies the disjunct of $(M, N^*)$ in Theorem~\ref{thm:normalisation}.

Since the normal form of \citet[Theorem~4.22]{DBLP:journals/jacm/EsparzaRS24} is equivalent to $\varphi$, so is the normal form of Theorem~\ref{thm:normalisation}.
\end{proof}

Every clause of Definition~\ref{def:body-rewrite} is at most as large as its $\rewMN{M}$-counterpart, so the length bound $2^{O(|\varphi|)}$ of \citet[Theorem~4.22]{DBLP:journals/jacm/EsparzaRS24} carries over to Theorem~\ref{thm:normalisation}.

\section{Supplementary Material for Section~\ref{sec:min-tight-completion}}\label{appendix:min-tight-completion}
\choplemma*
\begin{proof}
Throughout the proof, both sides evaluate their formula at position $p$. The left-hand side on the full trace $u = u[0..N]$, the right-hand side on the trace $u[0..N-1]$. Two elementary facts are used in every case.
First, a position $j$ of $u$ satisfies $\neg\last$ if and only if $j \leq N-1$, that is, if and only if $j$ is also a position of $u[0..N-1]$.
Second, $u$ and $u[0..N-1]$ carry the same letter at every position $j \leq N-1$.
We proceed by structural induction on $\sigma \in \Sigma_1$, at all positions $p$ at once. 
\begin{itemize}
\item $\sigma \in \{\ltlftrue, \ltlffalse, a, \neg a\}$.
Here $\chop(\sigma) = \sigma$, and both sides depend only on the letter at position $p$. Since $p \leq N-1$, that letter is the same in $u$ and in $u[0..N-1]$.

\item $\sigma = \sigma_1 \land \sigma_2$ (and $\lor$ analogously).
Here $\chop(\sigma) = \chop(\sigma_1) \land \chop(\sigma_2)$. 
Both semantics evaluate a conjunction at position $p$ as the conjunction of the evaluations, so the claim follows from the I.H.\ applied to $\sigma_1$ and $\sigma_2$ at $p$.

\item  $\sigma = \ltlfX\sigma'$.
Here $\chop(\ltlfX\sigma') = \ltlfX(\neg\last \land \chop(\sigma'))$.

$(\Rightarrow)$. Assume that the left-hand side holds. The strong next provides the successor position $p+1$ and asserts both conjuncts there. The conjunct $\neg\last$ tells us that $p+1 \leq N-1$, so $p+1$ is a position of the shortened trace. The induction hypothesis at $p+1$ turns the conjunct $\chop(\sigma')$ into $u[p+1..N-1] \modelsfin \sigma'$. Since $p+1 \leq N-1$, position $p$ has a successor inside the shortened trace as well. Together this gives $u[p..N-1] \modelsfin \ltlfX\sigma'$.

$(\Leftarrow)$. Assume that the right-hand side holds. The strong next on the shortened trace provides $p+1 \leq N-1$ and $u[p+1..N-1] \modelsfin \sigma'$. The induction hypothesis at $p+1$ gives $u[p+1..N] \modelsfin \chop(\sigma')$. From $p+1 \leq N-1$ we also obtain $\neg\last$ at $p+1$ on the full trace, where the successor $p+1$ certainly exists. Together this gives the left-hand side.

\item $\sigma = \ltlfF\sigma'$. Here the rewrite is $\chop(\ltlfF\sigma') = \ltlfF(\neg\last \land \chop(\sigma'))$. On both sides, satisfaction requires a position at which the operand $\sigma'$ holds, and we show that the same positions can be chosen on both sides.

($\Rightarrow$). The left-hand side provides a position $j \in [p, N]$ at which $\neg\last \land \chop(\sigma')$ holds. The conjunct $\neg\last$ at $j$ tells us that $j \leq N-1$, so $j$ exists in the shortened trace. The induction hypothesis at $j$ turns $\chop(\sigma')$ at $j$ into $\sigma'$ at $j$ on the shortened trace. Hence the same $j$ establishes $u[p..N-1] \modelsfin \ltlfF\sigma'$.

($\Leftarrow$). The right-hand side provides a position $j \in [p, N-1]$ with $u[j..N-1] \modelsfin \sigma'$. Then $j \leq N-1$, because $j$ is a position of the shortened trace. The induction hypothesis at $j$ gives $\chop(\sigma')$ at $j$ on the full trace, and $j \leq N-1$ gives $\neg\last$ at $j$. Hence the same $j$ establishes the left-hand side.

\item $\sigma = \sigma_1 \ltlfU \sigma_2$. 
Here the rewrite is $\chop(\sigma_1 \ltlfU \sigma_2) = \chop(\sigma_1) \ltlfU (\neg\last \land \chop(\sigma_2))$.

($\Rightarrow$) The left-hand side provides a position $j$ at which $\neg\last \land \chop(\sigma_2)$ holds, with $\chop(\sigma_1)$ at every $i \in [p, j)$. The conjunct $\neg\last$ at $j$ gives $j \leq N-1$. Every waiting position satisfies $i < j \leq N-1$, so all positions involved exist in the shortened trace. The induction hypothesis at $j$ and at every $i$ turns $\chop(\sigma_2)$ into $\sigma_2$ and $\chop(\sigma_1)$ into $\sigma_1$ on the shortened trace. Hence the same $j$ establishes $u[p..N-1] \modelsfin \sigma_1 \ltlfU \sigma_2$.

($\Leftarrow$) The right-hand side provides a position $j \in [p, N-1]$ with $\sigma_2$ at $j$ and $\sigma_1$ at every $i \in [p, j)$, and all these positions are at most $N-1$. The induction hypothesis at each of them transfers both requirements to the full trace, and $j \leq N-1$ gives $\neg\last$ at $j$. Hence the same $j$ establishes the left-hand side.

\item $\sigma = \sigma_1 \ltlfM \sigma_2$.
The rewrite is $\chop(\sigma_1\ltlfM\sigma_2) = (\neg\last \land \chop(\sigma_1)) ~\ltlfM~ \chop(\sigma_2)$.

($\Rightarrow$) The left-hand side provides a position $j$ at which $\neg\last \land \chop(\sigma_1)$ holds, with $\chop(\sigma_2)$ at every $i \in [p, j]$. The conjunct $\neg\last$ at $j$ gives $j \leq N-1$, and every $i$ satisfies $i \leq j \leq N-1$, so all positions involved exist in the shortened trace. The induction hypothesis at $j$ and at every $i$ turns $\chop(\sigma_1)$ into $\sigma_1$ and $\chop(\sigma_2)$ into $\sigma_2$ on the shortened trace. Hence the same $j$ establishes $u[p..N-1] \modelsfin \sigma_1 \ltlfM \sigma_2$.

($\Leftarrow$) The right-hand side provides a position $j \in [p, N-1]$ with $\sigma_1$ at $j$ and $\sigma_2$ at every $i \in [p, j]$, and all these positions are at most $N-1$. The induction hypothesis at each of them transfers both requirements to the full trace, and $j \leq N-1$ gives $\neg\last$ at $j$. Hence the same $j$ establishes the left-hand side.
\end{itemize}
\end{proof}

\tightcompllemma*

\begin{proof}
Recall that $\tight(\sigma) = \finX(\sigma) \land (\last \lor \neg\chop(\finX(\sigma)))$, and write $u = w[p..n]$. If $c_\sigma(p) = \infty$, then no infix $w[p..m]$ satisfies $\finX(\sigma)$, so the first conjunct of $\tight(\sigma)$ fails on every $u$ and both sides of the claim are false for every $n \geq p$. Therefore,
assume from now on that $c_\sigma(p)$ is finite. We unfold the definition in four steps.

\begin{itemize}
\item The disjunct $\last$ holds at position $0$ of $u$ exactly when $n = p$. Hence
$
    u \modelsfin \tight(\sigma)
    \iff
    u \modelsfin \finX(\sigma)
    \text{ and }
    \bigl( n = p \text{ or } u \modelsfin \neg\chop(\finX(\sigma)) \bigr).
$

\item For $n \geq p+1$, Lemma~\ref{lem:chop} applied to $u$ at position $0$ states that $u \modelsfin \chop(\finX(\sigma))$ holds exactly when $w[p..n-1] \modelsfin \finX(\sigma)$ holds. Negating both sides, the condition above becomes
$u \modelsfin \finX(\sigma)
\text{ and }
\bigl( n = p \text{ or } w[p..n-1] \not\modelsfin \finX(\sigma) \bigr)$.

\item  Lemma~\ref{lem:monotonicity} turns both satisfaction conditions into threshold comparisons. We have $u \modelsfin \finX(\sigma)$ exactly when $n \geq c_\sigma(p)$, and, for $n \geq p+1$, $w[p..n-1] \not\modelsfin \finX(\sigma)$ exactly when $n - 1 < c_\sigma(p)$. The condition becomes
$n \geq c_\sigma(p)
\text{ and }
\bigl( n = p \text{ or } n - 1 < c_\sigma(p) \bigr)$.

\item This arithmetic condition is equivalent to $n = c_\sigma(p)$. If $n = c_\sigma(p)$, then $n \geq c_\sigma(p)$ holds, and either $n = p$ or $n - 1 = c_\sigma(p) - 1 < c_\sigma(p)$. Conversely, if $n \geq c_\sigma(p)$ and $n - 1 < c_\sigma(p)$, then $c_\sigma(p) \leq n \leq c_\sigma(p)$, so $n = c_\sigma(p)$. If instead $n \geq c_\sigma(p)$ and $n = p$, then $c_\sigma(p) \leq p$, and since $c_\sigma(p) \geq p$ we again obtain $n = p = c_\sigma(p)$.
\end{itemize}
\end{proof}

\begin{corollary}\label{corollary:E}
For every \ltl formula $\sigma \in \Sigma_1$ and every $w \in \Sigma^\omega$,
$w \modelsinf \sigma \iff w \modelsinf \exists(\tight(\sigma))$.
\end{corollary}
\noindent Intuitively, $\sigma$ holds exactly when the minimal completion point exists.
\begin{proof}
    By Lemmas~\ref{lem:2-3} and \ref{lem:2-4}, $w \modelsinf \sigma \iff c_\sigma < \infty$. By Lemma~\ref{lemma:tight-compl}, a prefix satisfying $\tight(\sigma)$ exists iff $c_\sigma < \infty$, namely $n = c_\sigma$ and no other.
\end{proof}

\begin{corollary}\label{corollary:A}
For every \ltl formula $\pi \in \Pi_1$ and every $w \in \Sigma^\omega$,
$w \modelsinf \pi \iff w \modelsinf \forall(\neg \tight(\neg \pi))$.
\end{corollary}
\noindent Intuitively, $\pi$ holds exactly when no prefix ever completes a violation.
\begin{proof}
    $w \modelsinf \forall (\neg \tight(\neg \pi))$ iff no prefix satisfies $\tight(\neg \pi)$, iff $w \not \modelsinf \neg \pi$ (Corollary~\ref{corollary:E} applied to $\neg \pi \in \Sigma_1$), iff $w \modelsinf \pi$.
\end{proof}

\existsforalltheorem*
\begin{proof}
    Since $\ltlfG(\neg \tight(\neg \pi)) \equiv \neg \ltlfF(\tight(\neg \pi))$ on finite traces, $w \modelsinf \exists \forall(\ltlfG(\neg \tight(\neg \pi))) \iff w \not \modelsinf \forall \exists (\ltlfF(\tight(\neg \pi)))$. By Theorem~\ref{theorem:AE} applied to $\neg \pi \in \Sigma_1$, the latter holds iff $w \not \modelsinf \ltlfG \ltlfF(\neg \pi)$, which is $w \modelsinf \ltlfF \ltlfG (\pi)$.
\end{proof}

\section{Supplementary Material for Section~\ref{sec:sigma2}}\label{appendix:sigma2}

The proofs of Lemmas~\ref{lem:proj} and~\ref{lem:body-stability} and Theorem~\ref{theorem:pi2} are all under the context of Section~\ref{sec:sigma2}.
A disjunct of the normal form given by the limit-behavior $(M, N)$, an infinite trace $w$ exhibiting it, the persistence index $T$ of Definition~\ref{def:stab-index}, and the projection of Definition~\ref{def:proj}. 
The supporting definitions and lemmas appear where the proofs need them. 
We restate the two lemmas (Lemma~\ref{lem:proj} and Lemma~\ref{lem:body-stability}) from the main text and prove them, and close with Theorem~\ref{theorem:pi2}, a recurrence variant of Theorem~\ref{theorem:sigma2}.

The induction below needs the convergence of Definition~\ref{def:convergent} at every position of the trace, not only at position $0$.

\begin{definition}[Convergent at a position]\label{def:stable}
A formula $\varphi$ is convergent at a position $k$ if there is a bound $n_0 \geq k$ such that, for every $n \geq n_0$, $w[k..n] \modelsfin \finX(\varphi)$ holds if and only if $w[k...] \modelsinf \varphi$. Convergence at position $0$ is the convergence of Definition~\ref{def:convergent}.
\end{definition}

\begin{definition}[Soundness past a bound]\label{def:soundness-past}
A formula $\varphi$ is sound past $B$ if, for every position $k \geq B$, $w[k..n] \modelsfin \finX(\varphi)$ for some $n \geq k$ implies $w[k...] \modelsinf \varphi$. Equivalently, for every position $k \geq B$, $w[k...] \not \modelsinf \varphi$ implies $w[k..n] \not \modelsfin \finX(\varphi)$ for every $n \geq k$.
\end{definition}
Intuitively, convergence says the finite evaluations tell the truth from some point on, and soundness past $B$ says a true finite evaluation from $B$ on is never wrong.

\begin{corollary}\label{cor:sig1_restated}
Let $\sigma \in \Sigma_1$. Then (i) $\sigma$ is convergent at every position $k$, with the completion index $c_\sigma(k)$ a bound of convergence at $k$ whenever $w[k...] \modelsinf \sigma$, and (ii) $\sigma$ is sound past 0.
\end{corollary}
A co-safety formula is the easy case, since by Lemma~\ref{lem:monotonicity} once $w[k..n] \modelsfin \finX(\sigma)$ is true it stays true as $n$ grows, and by Lemma~\ref{lem:2-3} it is never true at a position where $\sigma$ fails.

\begin{proof}
    Fix a position $k$. If $w[k...] \modelsinf \sigma$ then $c_{\sigma}(k)$ is finite by Lemma~\ref{lem:2-4}, and by Lemma~\ref{lem:monotonicity} $w[k..n] \modelsfin \finX(\sigma)$ holds iff $n \geq c_{\sigma}(k)$, which is convergence at $k$ with bound $c_{\sigma}(k)$.
    If $w[k...] \not \modelsinf \sigma$ then by Lemma~\ref{lem:2-3} $w[k..n] \not \modelsfin \finX(\sigma)$ for every $n \geq k$, which is convergence at $k$ with bound $k$, and soundness past 0.
\end{proof}

The outermost operator of each safety leaf of the body is a $\nu$-operator, and on such a formula $\finX$ and $\finN$ coincide by Definition~\ref{def:fin-c}, since its operands enter the $\ltlfN$-context whatever the incoming context, so the leaves are governed by the weak reinterpretation.
We therefore record its behavior on $\Pi_1$ formulas, which we require once in the proof of Lemma~\ref{lem:proj}, and with both of its parts in the proof of Lemma~\ref{lem:body-stability}.

\begin{lemma}\label{lem:pi-1-formulas}
Let $\pi \in \Pi_1$. For every position $k$ of $w$: (i) If $w[k...] \modelsinf \pi$, then $w[k..n] \modelsfin \finN(\pi)$ for every $n \geq k$.
(ii) If $w[k...] \not \modelsinf \pi$, then there is a bound $n_0 \geq k$, depending on $\pi$ and on $k$, with $w[k..n] \not \modelsfin \finN(\pi)$ for every $n \geq n_0$.
\end{lemma}
\begin{proof}
    On a $\Pi_1$ formula no operator switches the $\fin_c$ (Definition~\ref{def:fin-c}) context to $c = \ltlfX$. So $\finN$ carries $\land$, $\lor$, $\ltlfG$, $\ltlfW$, $\ltlfR$ to themselves and every next to the weak next, in every case translating the operands again by $\finN$. We prove both parts with a structural induction on $\pi$, at all positions $k$ at once.

    (i) Suppose $w[k...] \modelsinf \pi$.
    \begin{itemize}

    \item For a literal, for $\ltlftrue$, and for $\ltlffalse$, $w[k..n] \modelsfin \finN(\pi)$ does not depend on $n$ and equals $w[k...] \modelsinf \pi$, which is true.

    \item For $\pi' \land \pi''$, $w[k...] \modelsinf \pi' \land \pi''$ gives $w[k...] \modelsinf \pi'$ and $w[k...] \modelsinf \pi''$. By the induction hypothesis $w[k..n] \modelsfin \finN(\pi')$ and $w[k..n] \modelsfin \finN(\pi'')$ hold for every $n \geq k$, so $w[k..n] \modelsfin \finN(\pi' \land \pi'')$ holds for every $n \geq k$.

    \item For $\pi' \lor \pi''$, $w[k...] \modelsinf \pi' \lor \pi''$ gives $w[k...] \modelsinf \pi'$ or $w[k...] \modelsinf \pi''$, and the induction hypothesis on the true disjunct gives $w[k..n] \modelsfin \finN(\pi' \lor \pi'')$ for every $n \geq k$.

    \item For a next $\ltlfX\pi'$, $w[k...] \modelsinf \ltlfX\pi'$ gives $w[k+1...] \modelsinf \pi'$. On the prefix with $n = k$, position $k$ is the last one, where the weak next is vacuously true. On a prefix with $n > k$, the weak next reduces to $w[k+1..n] \modelsfin \finN(\pi')$, true by the induction hypothesis at $k+1$.

    \item For $\ltlfG\pi'$, $w[k...] \modelsinf \ltlfG\pi'$ gives $w[j...] \modelsinf \pi'$ at every position $j \geq k$. The finite clause of $\ltlfG$ on $w[k..n]$ requires $w[j..n] \modelsfin \finN(\pi')$ at every position $j \in [k, n]$, and the induction hypothesis at each such position supplies it.

    \item For $\pi' \ltlfW \pi''$, recall $\pi' \ltlfW \pi'' \equiv (\pi' \ltlfU \pi'') \lor \ltlfG\pi'$, so $w[k...] \modelsinf \pi' \ltlfW \pi''$ gives two branches. In the until branch, there is a position $j \geq k$ with $w[j...] \modelsinf \pi''$ and with $w[j'...] \modelsinf \pi'$ at every $j' \in [k, j)$. If $j \leq n$, the induction hypothesis at $j$ and at every position of $[k, j)$ gives the finite until branch on $w[k..n]$. 
    If $j > n$, $w[j'...] \modelsinf \pi'$ at every position $j' \in [k, n]$, and the induction hypothesis gives the finite $\ltlfG\pi'$ branch. In the $\ltlfG\pi'$ branch, $w[j...] \modelsinf \pi'$ at every position $j \geq k$, and the induction hypothesis gives the finite $\ltlfG\pi'$ branch on every prefix.

    \item For $\pi' \ltlfR \pi''$, recall $\pi' \ltlfR \pi'' \equiv (\pi'' \ltlfU (\pi' \land \pi'')) \lor \ltlfG\pi''$, so $w[k...] \modelsinf \pi' \ltlfR \pi''$ gives two branches. In the release branch, there is a position $j \geq k$ with $w[j...] \modelsinf \pi'$ and with $w[j'...] \modelsinf \pi''$ at every position $j'$ of the inclusive range $[k, j]$. If $j \leq n$, the induction hypothesis at $j$ and at every position of the inclusive range $[k, j]$ gives the finite release branch on $w[k..n]$. 
    If $j > n$, $w[j'...] \modelsinf \pi''$ at every position $j' \in [k, n]$, and the induction hypothesis gives the finite $\ltlfG\pi''$ branch. In the $\ltlfG\pi''$ branch, the induction hypothesis gives the finite $\ltlfG\pi''$ branch on every prefix.
\end{itemize}

(ii) Suppose $w[k...] \not\modelsinf \pi$. Every case names its bound.
\begin{itemize}
  \item For a literal, for $\ltlftrue$, and for $\ltlffalse$, $w[k..n] \modelsfin \finN(\pi)$ equals $w[k...] \modelsinf \pi$, which is false, and the bound is $k$.

  \item For $\pi' \land \pi''$, since $w[k...] \not\modelsinf \pi' \land \pi''$ some conjunct is false, $w[k...] \not\modelsinf \pi'$ say, and the bound $b$ of $\pi'$ from the induction hypothesis serves. For every $n \geq b$, $w[k..n] \not\modelsfin \finN(\pi')$, so $w[k..n] \not\modelsfin \finN(\pi' \land \pi'')$.

  \item For $\pi' \lor \pi''$, $w[k...] \not\modelsinf \pi' \lor \pi''$ gives $w[k...] \not\modelsinf \pi'$ and $w[k...] \not\modelsinf \pi''$, and the larger of their two bounds serves.

  \item For a next $\ltlfX\pi'$, $w[k...] \not\modelsinf \ltlfX\pi'$ gives $w[k+1...] \not\modelsinf \pi'$, and the induction hypothesis gives a bound $b$ at $k+1$. On every prefix $w[k..n]$ with $n \geq b$ and $n \geq k+1$, position $k$ is not the last one, so the weak next reduces to $w[k+1..n] \modelsfin \finN(\pi')$, which is false. The bound is the larger of $b$ and $k+1$.

  \item For $\ltlfG\pi'$, $w[k...] \not\modelsinf \ltlfG\pi'$ gives a position $j \geq k$ with $w[j...] \not\modelsinf \pi'$, and the induction hypothesis gives a bound $b$ at $j$. On every prefix $w[k..n]$ with $n \geq j$ and $n \geq b$, the finite clause of $\ltlfG$ at $k$ requires $w[j..n] \modelsfin \finN(\pi')$ at position $j$, and that is false. The bound is the larger of $j$ and $b$.

  \item For $\pi' \ltlfW \pi''$, since $w[k...] \not\modelsinf \pi' \ltlfW \pi''$ both branches of $(\pi' \ltlfU \pi'') \lor \ltlfG\pi'$ fail at $k$. Since $\ltlfG\pi'$ fails, there is a least position $l \geq k$ with $w[l...] \not\modelsinf \pi'$. Now $w[j...] \not\modelsinf \pi''$ at every position $j \in [k, l]$. Otherwise $w[j...] \modelsinf \pi''$ for some $j \in [k, l]$,
  and by the minimality of $l$ we have $w[j'...] \modelsinf \pi'$ at every position $j' \in [k, j)$, a range inside $[k, l)$ since $j \leq l$, so the until branch would hold at $k$.
  The induction hypothesis now gives a bound for $\pi'$ at $l$ and one for $\pi''$ at every position of $[k, l]$. The bound of the case, $n_0$, is the largest of these finitely many bounds and $l$. Fix a prefix $w[k..n]$ with $n \geq n_0$. 
  The finite $\ltlfG\pi'$ branch fails, since $w[l..n] \not\modelsfin \finN(\pi')$ and $l \leq n$.
  A position $j \in [k, l]$ cannot satisfy the finite until branch, since $w[j..n] \not\modelsfin \finN(\pi'')$. A position $j > l$ cannot satisfy it either, since the finite until branch needs $w[m..n] \modelsfin \finN(\pi')$ at every position $m \in [k, j)$, and $l$ is one of them. So $w[k..n] \not\modelsfin \finN(\pi' \ltlfW \pi'')$.

  \item For $\pi' \ltlfR \pi''$, since $w[k...] \not\modelsinf \pi' \ltlfR \pi''$ both branches of $(\pi'' \ltlfU (\pi' \land \pi'')) \lor \ltlfG\pi''$ fail at $k$. Since $\ltlfG\pi''$ fails, there is a least position $l \geq k$ with $w[l...] \not\modelsinf \pi''$. Now $w[j...] \not\modelsinf \pi'$ at every position $j \in [k, l)$. Otherwise $w[j...] \modelsinf \pi'$ for some $j \in [k, l)$, and by the minimality of $l$ we have $w[j'...] \modelsinf \pi''$ at every position $j'$ of the inclusive range $[k, j]$, a range inside $[k, l)$ since $j < l$, so the release branch would hold at $k$. The induction hypothesis now gives a bound for $\pi''$ at $l$ and one for $\pi'$ at every position of $[k, l)$. The bound of the case, $n_0$, is the largest of these finitely many bounds and $l$. Fix a prefix $w[k..n]$ with $n \geq n_0$. 
  The finite $\ltlfG\pi''$ branch fails, since $w[l..n] \not\modelsfin \finN(\pi'')$ and $l \leq n$.
  A position $j \in [k, l)$ cannot satisfy the finite release branch, since $w[j..n] \not\modelsfin \finN(\pi')$. A position $j \geq l$ cannot satisfy it either, since the finite release branch needs $w[m..n] \modelsfin \finN(\pi'')$ at every position $m$ of the inclusive range $[k, j]$, and $l$ is one of them. So $w[k..n] \not\modelsfin \finN(\pi' \ltlfR \pi'')$.
\end{itemize}

\end{proof}

\projlemma*
Past $T$ every safety leaf is true, so a body formula behaves exactly like the $\Sigma_1$ formula obtained by replacing its leaves with $\ltlftrue$.
\begin{proof}
    We prove by structural induction over the body grammar, at all positions $i \geq T$ at once.
    \begin{itemize}
        \item For a literal, for $\ltlftrue$, and for $\ltlffalse$, the projection leaves the formula unchanged, and the two sides are the same formula.

        \item For a safety leaf $\ltlfG(\psi\fgrw{M})$, $w[i...] \modelsinf \ltlfG(\psi\fgrw{M})$ by Definition~\ref{def:stab-index}. Its operand $\psi\fgrw{M}$ is $\Pi_1$ by Definition~\ref{def:fg-rewrite} and $\Pi_1$ is closed under $\ltlfG$, so the safety leaf is itself a $\Pi_1$ formula, and $\finX$ and $\finN$ coincide on formulas whose outermost operator is a $\nu$-operator (Definition~\ref{def:fin-c}), so Lemma~\ref{lem:pi-1-formulas}(i) applies and makes $w[i..n] \modelsfin \finX(\ltlfG(\psi\fgrw{M}))$ true on every prefix. The projection replaces the safety leaf by $\ltlftrue$, and $w[i...] \modelsinf \ltlftrue$ and $w[i..n] \modelsfin \finX(\ltlftrue)$ are also true.

        \item For $\varphi' \land \varphi''$, $w[i...] \modelsinf \varphi' \land \varphi''$ holds exactly when $w[i...] \modelsinf \varphi'$ and $w[i...] \modelsinf \varphi''$, and $w[i..n] \modelsfin \finX(\varphi' \land \varphi'')$ holds exactly when $w[i..n] \modelsfin \finX(\varphi')$ and $w[i..n] \modelsfin \finX(\varphi'')$. The same two descriptions hold for $\proj(\varphi') \land \proj(\varphi'')$, so the induction hypothesis at $i$ transfers both. For $\varphi' \lor \varphi''$ the same argument runs with the truth of some operand in place of the truth of both.

        \item For $\ltlfX\varphi'$, $w[i...] \modelsinf \ltlfX\varphi'$ equals $w[i+1...] \modelsinf \varphi'$, and $w[i..n] \modelsfin \finX(\ltlfX\varphi')$ is given by the clause of the strong next, false for $n = i$ and equal to $w[i+1..n] \modelsfin \finX(\varphi')$ for $n > i$. Both descriptions hold for $\ltlfX\proj(\varphi')$ as well, so the induction hypothesis at $i+1$ transfers $w[i...] \modelsinf \cdot$ and every $w[i..n] \modelsfin \finX(\cdot)$ with $n > i$, and for $n = i$ both $w[i..i] \modelsfin \finX(\ltlfX\varphi')$ and $w[i..i] \modelsfin \finX(\ltlfX\proj(\varphi'))$ are false.

        \item For $\ltlfF\varphi'$, $w[i...] \modelsinf \ltlfF\varphi'$ holds exactly when $w[j...] \modelsinf \varphi'$ for some position $j \geq i$, and $w[i..n] \modelsfin \finX(\ltlfF\varphi')$ holds exactly when $w[j..n] \modelsfin \finX(\varphi')$ for some position $j \in [i, n]$. Every such $j$ satisfies $j \geq T$, so the induction hypothesis at $j$ transfers both descriptions to $\ltlfF\proj(\varphi')$.

        \item For $\varphi' \ltlfU \varphi''$, $w[i...] \modelsinf \varphi' \ltlfU \varphi''$ holds exactly when $w[j...] \modelsinf \varphi''$ for some position $j \geq i$ while $w[m...] \modelsinf \varphi'$ at every position $m \in [i, j)$, and $w[i..n] \modelsfin \finX(\varphi' \ltlfU \varphi'')$ holds exactly when $w[j..n] \modelsfin \finX(\varphi'')$ for some position $j \in [i, n]$ while $w[m..n] \modelsfin \finX(\varphi')$ at every position $m \in [i, j)$. All positions involved are $\geq T$, so the induction hypothesis applies at each of them and transfers both descriptions to $\proj(\varphi') \ltlfU \proj(\varphi'')$.

        \item For $\varphi' \ltlfM \varphi''$, $w[i...] \modelsinf \varphi' \ltlfM \varphi''$ holds exactly when $w[j...] \modelsinf \varphi'$ for some position $j \geq i$ while $w[m...] \modelsinf \varphi''$ at every position $m$ of the inclusive range $[i, j]$, and $w[i..n] \modelsfin \finX(\varphi' \ltlfM \varphi'')$ holds by the same condition with $j \in [i, n]$ and $w[\cdot..n] \modelsfin \finX(\cdot)$ in place of $w[\cdot...] \modelsinf \cdot$. The induction hypothesis again applies at every position involved and transfers both descriptions to $\proj(\varphi') \ltlfM \proj(\varphi'')$.
    \end{itemize}
    
\end{proof}

\begin{corollary}\label{cor:body-sound}
Every body formula is sound past $T$.
\end{corollary}
The soundness of $\Sigma_1$ carries over to the body formulas through the projection.
\begin{proof}
    Suppose $w[i...] \not\modelsinf \varphi$ for a body formula $\varphi$ at some $i \geq T$. By the $\modelsinf$ equivalence of Lemma~\ref{lem:proj}, $w[i...] \not\modelsinf \proj(\varphi)$. By Corollary~\ref{cor:sig1_restated}, $\proj(\varphi) \in \Sigma_1$ is sound past $0$, so $w[i..n] \not\modelsfin \finX(\proj(\varphi))$ for every $n \geq i$. By the $\modelsfin$ equivalence of Lemma~\ref{lem:proj}, $w[i..n] \not\modelsfin \finX(\varphi)$ for every $n \geq i$.
\end{proof}

\begin{lemma}\label{lem:closure}
    If $\varphi$ and $\varphi'$ are convergent at every position, then so are $\varphi \land \varphi'$, $\varphi \lor \varphi'$ and $\ltlfX\varphi$.
\end{lemma}
These connectives inspect their operands at a bounded offset, so the operands' bounds simply combine.
\begin{proof}
Fix a position $k$. $w[k..n] \modelsfin \finX(\varphi \land \varphi')$ is the conjunction of $w[k..n] \modelsfin \finX(\varphi)$ and $w[k..n] \modelsfin \finX(\varphi')$, and $w[k...] \modelsinf \varphi \land \varphi'$ the conjunction of $w[k...] \modelsinf \varphi$ and $w[k...] \modelsinf \varphi'$, so past the larger of the two bounds at $k$, $w[k..n] \modelsfin \finX(\varphi \land \varphi')$ equals $w[k...] \modelsinf \varphi \land \varphi'$. Disjunction is identical. 
$w[k..n] \modelsfin \finX(\ltlfX\varphi)$ is false for $n = k$ and equals $w[k+1..n] \modelsfin \finX(\varphi)$ for $n > k$, while $w[k...] \modelsinf \ltlfX\varphi$ equals $w[k+1...] \modelsinf \varphi$, so past the larger of $k+1$ and the bound of $\varphi$ at $k+1$, $w[k..n] \modelsfin \finX(\ltlfX\varphi)$ equals $w[k...] \modelsinf \ltlfX\varphi$.
\end{proof}

\begin{lemma}\label{lem:until}
    Let $\varphi$ and $\varphi'$ be convergent at every position and sound past a common bound $B$. Then $\varphi \ltlfU \varphi'$ is convergent at every position.
\end{lemma}
Until is the difficult case, as it inspects unboundedly many positions while convergence bounds only finitely many of them. Soundness past a common bound is what rules out a false witness in the unbounded tail.
\begin{proof}
    Fix a position $i$, $w[i..n] \modelsfin \finX(\varphi \ltlfU \varphi')$ holds exactly when $w[j..n] \modelsfin \finX(\varphi')$ for some position $j \in [i, n]$ while $w[m..n] \modelsfin \finX(\varphi)$ at every position $m \in [i, j)$.

    Suppose $w[i...] \modelsinf \varphi \ltlfU \varphi'$. 
    Fix a position $j$ with $w[j...] \modelsinf \varphi'$ and $w[j'...] \modelsinf \varphi$ at every $j' \in [i, j)$. 
    These are finitely many positions, and since $\varphi$ and $\varphi'$ are convergent at every position (Definition~\ref{def:stable}), each has a bound past which its finite evaluation is true.
    On every prefix past $j$ and past the largest of these bounds, $w[j..n] \modelsfin \finX(\varphi')$ and every $w[j'..n] \modelsfin \finX(\varphi)$ are true, so $w[i..n] \modelsfin \finX(\varphi \ltlfU \varphi')$ is true.

    Suppose $w[i...] \not\modelsinf \varphi \ltlfU \varphi'$. Let $l$ be the least position $\geq i$ with $w[l...] \not\modelsinf \varphi$, and $l = \infty$ if there is none, in which case $[i, l]$ reads as $[i, \infty)$ and there is no $j > l$, so the case below can be skipped.
    Then $w[m...] \not\modelsinf \varphi'$ at every position $m \in [i, l]$. 
    Were $w[j...] \modelsinf \varphi'$ for some $j \in [i, l]$, then with $w[j'...] \modelsinf \varphi$ at every $j' \in [i, j)$, a range inside $[i, l)$ since $j \leq l$, $w[i...] \modelsinf \varphi \ltlfU \varphi'$ would follow.

    Now $w[i..n] \modelsfin \finX(\varphi \ltlfU \varphi')$ requires a position $j \in [i, n]$ with $w[j..n] \modelsfin \finX(\varphi')$ and $w[m..n] \modelsfin \finX(\varphi)$ at every $m \in [i, j)$. 
    A position $j \leq l$ needs $w[j..n] \modelsfin \finX(\varphi')$, while $w[j...] \not\modelsinf \varphi'$. 
    A position $j > l$ needs $w[l..n] \modelsfin \finX(\varphi)$, as $l \in [i, j)$, while $w[l...] \not\modelsinf \varphi$. 
    So each $j$ requires the finite evaluation of an operand to hold at a position where the operand does not hold under $\modelsinf$, at $j$ for $\varphi'$ when $j \leq l$ and at $l$ for $\varphi$ when $j > l$.

    It remains to make each of these finite evaluations false. 
    At each of these positions before $B$, of which there are finitely many, $\varphi$ and $\varphi'$ are convergent there (Definition~\ref{def:stable}), so the operand has a bound past which its finite evaluation is false.
    At each position at or beyond $B$, $\varphi$ and $\varphi'$ are sound past $B$ (Definition~\ref{def:soundness-past}), so the operand's finite evaluation is false on every prefix, with no bound. 
    This handles the positions $\geq B$, unboundedly many when $l = \infty$. 
    Let $B'$ be the largest of $i$ and these finitely many bounds from the positions before $B$. 
    For every $n \geq B'$, no position $j \in [i, n]$ has both $w[j..n] \modelsfin \finX(\varphi')$ and $w[m..n] \modelsfin \finX(\varphi)$ at every $m \in [i, j)$, so $w[i..n] \not\modelsfin \finX(\varphi \ltlfU \varphi')$.
\end{proof}

\bodystablemma*
\begin{proof}
We prove the stronger statement that every body formula is convergent at every position (Definition~\ref{def:stable}). 
The lemma is the case of position $0$. The proof is a structural induction over the body grammar, the until case taken before the $\ltlfF$ and $\ltlfM$ cases, which reduce to until formulas.

\begin{itemize}

\item For a literal, for $\ltlftrue$, and for $\ltlffalse$, $w[k..n] \modelsfin \finX(\varphi)$ equals $w[k...] \modelsinf \varphi$ for every $n \geq k$ and every position $k$, which is convergence at every position (Definition~\ref{def:stable}).

\item For a safety leaf $\ltlfG(\psi\fgrw{M})$, the operand $\psi\fgrw{M}$ is $\Pi_1$ by Definition~\ref{def:fg-rewrite}, and $\Pi_1$ is closed under $\ltlfG$, so the safety leaf is itself a $\Pi_1$ formula, and $\finX$ and $\finN$ coincide on formulas whose outermost operator is a $\nu$-operator (Definition~\ref{def:fin-c}), so $w[k..n] \modelsfin \finX(\ltlfG(\psi\fgrw{M}))$ equals $w[k..n] \modelsfin \finN(\ltlfG(\psi\fgrw{M}))$ and Lemma~\ref{lem:pi-1-formulas} applies. If $w[k...] \modelsinf \ltlfG(\psi\fgrw{M})$, part~(i) makes $w[k..n] \modelsfin \finN(\ltlfG(\psi\fgrw{M}))$ true for every $n \geq k$. If $w[k...] \not\modelsinf \ltlfG(\psi\fgrw{M})$, part~(ii) makes it false from a bound onward. In both cases $w[k..n] \modelsfin \finX(\ltlfG(\psi\fgrw{M}))$ equals $w[k...] \modelsinf \ltlfG(\psi\fgrw{M})$ from some point on, which is convergence at every position (Definition~\ref{def:stable}).

\item For $\varphi' \land \varphi''$, the operands are convergent at every position by the induction hypothesis, and Lemma~\ref{lem:closure} concludes.

\item For $\varphi' \lor \varphi''$, the operands are convergent at every position by the induction hypothesis, and Lemma~\ref{lem:closure} concludes.

\item For $\ltlfX\varphi'$, the operand is convergent at every position by the induction hypothesis, and Lemma~\ref{lem:closure} concludes.

\item For $\varphi' \ltlfU \varphi''$, the operands are convergent at every position by the induction hypothesis. Both are body formulas, so Corollary~\ref{cor:body-sound} makes both sound past $T$, and Lemma~\ref{lem:until} concludes.

\item For $\ltlfF\varphi'$, we claim that $w[k...] \modelsinf \ltlfF\varphi'$ equals $w[k...] \modelsinf \ltlftrue \ltlfU \varphi'$ at every position $k$, and $w[k..n] \modelsfin \finX(\ltlfF\varphi')$ equals $w[k..n] \modelsfin \finX(\ltlftrue \ltlfU \varphi')$ on every prefix.

Under $\modelsinf$, $w[k...] \modelsinf \ltlfF\varphi'$ holds exactly when $w[j...] \modelsinf \varphi'$ for some $j \geq k$, and $w[k...] \modelsinf \ltlftrue \ltlfU \varphi'$ holds exactly when $w[j...] \modelsinf \varphi'$ for some $j \geq k$ while $w[m...] \modelsinf \ltlftrue$ at every $m \in [k, j)$. Since $w[m...] \modelsinf \ltlftrue$ at every position, the two coincide.

Under $\modelsfin$, $\finX(\ltlfF\varphi') = \ltlfF(\finX(\varphi'))$ and $\finX(\ltlftrue \ltlfU \varphi') = \ltlftrue \ltlfU \finX(\varphi')$ (Definition~\ref{def:fin-c}, as $\finX(\ltlftrue) = \ltlftrue$), sharing the operand $\finX(\varphi')$, so the same reasoning equates $w[k..n] \modelsfin \finX(\ltlfF\varphi')$ with $w[k..n] \modelsfin \finX(\ltlftrue \ltlfU \varphi')$ on every prefix.

$w[k...] \modelsinf \ltlfF\varphi'$ equals $w[k...] \modelsinf \ltlftrue \ltlfU \varphi'$ and $w[k..n] \modelsfin \finX(\ltlfF\varphi')$ equals $w[k..n] \modelsfin \finX(\ltlftrue \ltlfU \varphi')$, so by Definition~\ref{def:stable} $\ltlfF\varphi'$ is convergent at every position exactly when $\ltlftrue \ltlfU \varphi'$ is. Now $\ltlftrue$ is convergent at every position and sound past $0$ ($\ltlftrue \in \Sigma_1$, Corollary~\ref{cor:sig1_restated}), hence sound past $T$. The operand $\varphi'$ is convergent at every position by the induction hypothesis and, as a body formula, sound past $T$ by Corollary~\ref{cor:body-sound}. Lemma~\ref{lem:until} makes $\ltlftrue \ltlfU \varphi'$ convergent at every position, hence so is $\ltlfF\varphi'$.

\item For $\varphi' \ltlfM \varphi''$, we claim that $w[k...] \modelsinf \varphi' \ltlfM \varphi''$ equals $w[k...] \modelsinf \varphi'' \ltlfU (\varphi' \land \varphi'')$ at every position $k$, and $w[k..n] \modelsfin \finX(\varphi' \ltlfM \varphi'')$ equals $w[k..n] \modelsfin \finX(\varphi'' \ltlfU (\varphi' \land \varphi''))$ on every prefix.

Under $\modelsinf$, $w[k...] \modelsinf \varphi' \ltlfM \varphi''$ holds exactly when $w[j...] \modelsinf \varphi'$ for some $j \geq k$ while $w[m...] \modelsinf \varphi''$ at every $m \in [k, j]$, and $w[k...] \modelsinf \varphi'' \ltlfU (\varphi' \land \varphi'')$ holds exactly when $w[j...] \modelsinf \varphi' \land \varphi''$ for some $j \geq k$ while $w[m...] \modelsinf \varphi''$ at every $m \in [k, j)$. These coincide, as $w[j...] \modelsinf \varphi' \land \varphi''$ puts $\varphi'$ and $\varphi''$ at $j$, extending $w[m...] \modelsinf \varphi''$ from $[k, j)$ to $[k, j]$.

Under $\modelsfin$, $\finX(\varphi' \ltlfM \varphi'') = \finX(\varphi') \ltlfM \finX(\varphi'')$ and $\finX(\varphi'' \ltlfU (\varphi' \land \varphi'')) = \finX(\varphi'') \ltlfU (\finX(\varphi') \land \finX(\varphi''))$ (Definition~\ref{def:fin-c}), sharing the operands $\finX(\varphi')$ and $\finX(\varphi'')$, so the same reasoning equates $w[k..n] \modelsfin \finX(\varphi' \ltlfM \varphi'')$ with $w[k..n] \modelsfin \finX(\varphi'' \ltlfU (\varphi' \land \varphi''))$ on every prefix.

$w[k...] \modelsinf \varphi' \ltlfM \varphi''$ equals $w[k...] \modelsinf \varphi'' \ltlfU (\varphi' \land \varphi'')$ and $w[k..n] \modelsfin \finX(\varphi' \ltlfM \varphi'')$ equals $w[k..n] \modelsfin \finX(\varphi'' \ltlfU (\varphi' \land \varphi''))$, so by Definition~\ref{def:stable} $\varphi' \ltlfM \varphi''$ is convergent at every position exactly when $\varphi'' \ltlfU (\varphi' \land \varphi'')$ is. Now $\varphi'$ and $\varphi''$ are convergent at every position by the induction hypothesis, so $\varphi' \land \varphi''$ is convergent at every position by Lemma~\ref{lem:closure}. Both $\varphi''$ and $\varphi' \land \varphi''$ are body formulas, the grammar being closed under $\land$, so Corollary~\ref{cor:body-sound} makes both sound past $T$. Lemma~\ref{lem:until} makes $\varphi'' \ltlfU (\varphi' \land \varphi'')$ convergent at every position, hence so is $\varphi' \ltlfM \varphi''$.

\end{itemize}
    
\end{proof}

The persistence quantifier is not the only quantifier that reads the limit of Lemma~\ref{lem:body-stability}: the recurrence quantifier, applied to the weak reinterpretation, captures the body as well.

\begin{theorem}\label{theorem:pi2}
Given a disjunct of the $\Delta_2$ normalization with limit-behavior $(M, N)$ and body formula $\varphi\rewM{M,N}$, and a trace $w$ that exhibits $(M, N)$, we obtain
$$
w \modelsinf \varphi\rewM{M,N} \iff w \modelsinf \forall \exists(\finN(\varphi\rewM{M,N})).
$$
\end{theorem}
\begin{proof}
By Lemma~\ref{lem:body-stability} the body is convergent, so $w[0..n] \modelsfin \finX(\varphi[M,N])$ is constant in $n$ from some point on and equals $w \modelsinf \varphi[M,N]$.
The claimed equivalence replaces $\finX$ by the weak $\finN$, so we first give the $\finN$ sequence the same limit. $\finX$ and $\finN$ differ only at the nexts that lie below no $\ltlfF$, $\ltlfU$, $\ltlfM$ operator and no safety leaf of the body (Definition~\ref{def:fin-c}). Call these the \emph{leading nexts}, and let $d$ be the greatest number of them along one path of the body. We show that $\finX$ and $\finN$ agree on every prefix $w[0..n]$ with $n \geq d$.

Fix such a prefix. $\finX(\varphi[M,N])$ and $\finN(\varphi[M,N])$ are the same formula except that each leading next carries $\ltlfX$ in the first and $\ltlfN$ in the second. For $i < n$, $w[i..n] \modelsfin \ltlfX \psi$ and $w[i..n] \modelsfin \ltlfN \psi$ both hold exactly when $w[i+1..n] \modelsfin \psi$ (Section~\ref{sec:prelim}), so the strong and the weak next differ only at the last position $n$. A leading next lies under only Boolean connectives and at most $d - 1$ leading nexts, hence is evaluated at a position at most $d - 1 < n$, never the last. So $w[0..n] \modelsfin \finX(\varphi[M,N])$ equals $w[0..n] \modelsfin \finN(\varphi[M,N])$.

Hence for $n \geq d$ the sequence $n \mapsto w[0..n] \modelsfin \finN(\varphi[M,N])$ coincides with the one under $\finX$, so it too is constant from some point on and equals $w \modelsinf \varphi[M,N]$. The equivalence now follows in both directions.

($\Rightarrow$) Suppose $w \modelsinf \varphi[M,N]$. For all large enough $n$, $w[0..n] \modelsfin \finN(\varphi[M,N])$ equals $w \modelsinf \varphi[M,N]$ and so holds. Thus infinitely many prefixes satisfy $\finN(\varphi[M,N])$, which is $\forall\exists(\finN(\varphi[M,N]))$.

($\Leftarrow$) Suppose $\forall\exists(\finN(\varphi[M,N]))$, so infinitely many prefixes satisfy $\finN(\varphi[M,N])$. Some $w[0..n]$ with $n$ large enough among them has $w[0..n] \modelsfin \finN(\varphi[M,N])$ equal to $w \modelsinf \varphi[M,N]$, so $w \modelsinf \varphi[M,N]$.
\end{proof}

\section{Implementation}

A reference implementation of the full pipeline, built on top of the \spot
library \citep{DBLP:conf/cav/Duret-LutzRCRAS22}, is available online. \footnote{\url{https://github.com/weinhuber/ltl2ltlfplus}}
It covers the normalization, the minimal tight completion, and the translation of $\Sigma_2$ formulas.
The implementation goes beyond the plain construction and includes several
optimizations. 
In particular, for the widely used obligation
fragment~\cite{DBLP:journals/corr/abs-2605-12372} it skips both the
normalization and the tight completion and directly returns obligation formulas of \ltlfplus.

\end{document}